\documentclass[twoside]{article}

\usepackage{subfigure}

\usepackage{times}
\usepackage{soul}
\usepackage[small]{caption}
\usepackage{graphicx}
\usepackage{amsmath}
\usepackage{amssymb}
\usepackage{makecell}
\usepackage{mathtools}
\usepackage{amsthm}
\usepackage{dsfont}
\usepackage{multirow}
\usepackage[group-separator={,},group-minimum-digits={3}]{siunitx}
\usepackage{wrapfig}

\usepackage{pifont}%
\newcommand{\cmark}{\ding{51}}%

\usepackage{tikz}
\usetikzlibrary{positioning}

\newtheorem{theorem}{Theorem}

\newtheorem{lemma}{Lemma}

\newtheorem{remark}{Remark}
\theoremstyle{definition}

\newcommand{\defeq}{\vcentcolon=}
\newcommand{\R}{\mathbb{R}}
\newcommand{\x}{{\boldsymbol{x}}}
\newcommand{\w}{{\boldsymbol{w}}}
\newcommand{\z}{{\boldsymbol{z}}}

\newcommand{\indicator}{\mathds{1}}
\renewcommand{\hat}{\widehat}
\newcommand{\SD}{\mathrm{pair}}
\newcommand{\PN}{\mathrm{point}}
\newcommand{\CE}{\mathrm{clus}}

\newcommand{\Er}{{\mathrm{Er}}}

\DeclareMathOperator*{\E}{\mathds{E}}
\DeclareMathOperator*{\argmin}{arg\,min}
\DeclareMathOperator*{\argmax}{arg\,max}

\DeclareMathOperator{\sign}{sign}

\usepackage[accepted]{aistats2022}

\usepackage[utf8]{inputenc}
\usepackage[T1]{fontenc}
\usepackage{hyperref}
\usepackage{url}
\usepackage{booktabs}
\usepackage{amsfonts}
\usepackage{nicefrac}
\usepackage{microtype}
\usepackage{xcolor}

\hypersetup{
  colorlinks=true,
  linkcolor={red!50!black},
  citecolor={green!50!black},
  urlcolor={blue!80!black}
}

\usepackage{natbib}
\setcitestyle{round}

\begin{document}

\runningauthor{Bao, Shimada, Xu, Sato, Sugiyama}

\twocolumn[
\aistatstitle{Pairwise Supervision Can Provably Elicit a Decision Boundary}
\aistatsauthor{%
  Han Bao$^{1,2,*}$ \quad
  Takuya Shimada$^{1,2,*,\dagger}$ \quad
  Liyuan Xu$^{3}$\quad
  Issei Sato$^{1}$\quad
  Masashi Sugiyama$^{2,1}$ \\
  ${}^1$The University of Tokyo, Japan \\
  ${}^2$RIKEN AIP, Japan \\
  ${}^3$Gatsby Unit, UCL, UK
}
\aistatsaddress{
  \footnotesize ${}^*$ Equal contribution (correspondence to Han Bao: \texttt{\href{mailto:tsutsumi@ms.k.u-tokyo.ac.jp}{tsutsumi@ms.k.u-tokyo.ac.jp}}) \\
  \footnotesize ${}^\dagger$ The author is now with Preferred Networks, Inc., Japan.}
]

\setcounter{footnote}{0}

\begin{abstract}
Similarity learning is a general problem to elicit useful representations
by predicting the relationship between a pair of patterns.
This problem is related to various important preprocessing tasks
such as metric learning, kernel learning, and contrastive learning.
A classifier built upon the representations is expected to perform well in downstream classification; however, little theory has been given in literature so far and thereby the relationship between similarity and classification has remained elusive.
Therefore, we tackle a fundamental question:
can similarity information provably leads a model to perform well in downstream classification?
In this paper,
we reveal that a product-type formulation of similarity learning is strongly related to an objective of binary classification.
We further show that these two different problems are explicitly connected by an excess risk bound.
Consequently, our results elucidate that similarity learning is capable of solving binary classification by directly eliciting a decision boundary.
\end{abstract}
\section{Introduction}
\label{sec:introduction}

Similarity learning is a learning paradigm~\citep{kulis2013metric} that builds a pairwise model to predict whether given paired patterns are similar or dissimilar in the classes that they belong to.
We call such a pair of patterns \emph{pairwise supervision},
in contrast to ordinary \emph{pointwise supervision} which binds a class label to a single input pattern.
Pairwise supervision is commonly available in many domains
such as geographical analysis~\citep{wagstaff2001constrained},
chemical experiment~\citep{eisenberg2000protein},
click-through feedback~\citep{davis2007information},
computer vision~\citep{yan2006discriminative,wang2015unsupervised},
natural language processing~\citep{mikolov2013distributed},
and crowdsourcing~\citep{gomes2012crowdclustering}.
Notably, feature representations can be constructed from pairwise supervision when it is not straightforward to define meaningful features~\citep{chen2009similarity,wang2009theory,kar2011similarity}.
This is one of the reasons why similarity learning has been studied extensively---including
metric learning~\citep{xing2003distance,bilenko2004integrating,davis2007information,weinberger2009distance,bellet2012similarity,niu2014information},
kernel learning~\citep{cristianini2002kernel,bach2004multiple,lanckriet2004learning,li2009constrained,cortes2010two},
and $(\varepsilon,\gamma,\tau)$-good similarity~\citep{balcan2008theory,wang2009theory,kar2011similarity,bellet2012similarity}, with different notion of similarity and models.
In recent studies, a similarity model is trained that aligns with pairwise supervision to capture inherent structures of data~\citep{bellet2012similarity,mikolov2013distributed,niu2014information,logeswaran2018efficient,saunshi2019theoretical}.
The learned similarity model is expected to help downstream tasks.
Correspondingly, it has been widely used for various downstream tasks such as classification~\citep{cristianini2002kernel,balcan2008theory,hsu2018multi,saunshi2019theoretical,nozawa2019pac},
clustering~\citep{bromley1994signature,xing2003distance,davis2007information,weinberger2009distance},
model selection~\citep{lanckriet2004learning},
and one-shot learning~\citep{koch2015siamese}.

The early theoretical research provided error bounds on classification based on similarity-based features by assuming that a given similarity metric is $(\varepsilon,\gamma,\tau)$-good~\citep{balcan2008theory,wang2009theory} (see related work for the details).
Recently, it has been attempted to investigate the relationship between learned similarity models and downstream classification, in order to deal with more flexible data structures.
\citet{bellet2012similarity} proved that features based on a learned metric are linearly separable under the framework of $(\varepsilon,\gamma,\tau)$-good similarity.
\citet{saunshi2019theoretical} analyzed how features learned in contrastive learning are meaningful in downstream classification.
These results boil down to two-step learners, which first solve similarity learning then train classifiers.
However, the latter step often requires as many samples as the former step
because the feature space constructed from the similarity function often becomes high-dimensional (see \citet[Theorem~1]{bellet2012similarity} for details).

In this work, we pose a question on what formulation of similarity learning is directly connected to downstream classification and reveal that similarity learning with a model $f(\x) \cdot f(\x')$ has a monotonic relationship to binary classification with a classifier $f(\x)$.
This interrelation provides a new insight that \emph{a binary decision boundary can essentially be obtained with only pairwise supervision} up to label permutation.
The post-process determining correct class assignments once classes are separated becomes less label-demanding than the previous formulations~\citep{bellet2012similarity,saunshi2019theoretical}.
While it is rather straightforward to use pointwise supervision to determine correct class assignments, we further found that pairwise supervision is sufficient for this purpose given that we know the majority class.
Our results are notable in that:
(i)~we unravel that similarity learning enables us to implicitly elicit a binary decision boundary without any explicit training of classifiers,
and (ii)~the post-process is less costly in terms of pointwise supervision.
Specifically, we will see:
similarity learning is tied to the binary classification error up to label permutation (Section~\ref{sec:clustering_error}).
The post-process to determine correct class assignments is discussed (Section~\ref{sec:class_assignment}).
As a by-product, we come across a training method of binary classifiers with only pairwise supervision (Section~\ref{sec:method}).
A finite-sample excess risk bound is established to connect similarity learning to binary classification (Section~\ref{sec:analysis}).
This theoretical finding is numerically demonstrated (Section~\ref{sec:experiments}).

\begin{remark}[Multi-class case]
    Despite that our main result (Theorem~\ref{theo:relationship}) is limited to the binary case,
    we can apply our training method (described in Section~\ref{sec:method}) in the multi-class case by the one-vs-rest approach:
    given $C$ classes, our training method can provide an one-vs-rest classifier for class $i \in [C]$ from pairwise supervision treating $[C] \setminus \{i\}$ as a single class.
    \citet{kar2011similarity} took the same one-vs-rest approach to first automatically construct feature representations from similarity information, and then train the one-vs-rest multi-class classifier with pointwise labels.
    This approach is valuable in the domains where samples are not immediately accessible in a Euclidean space yet sophisticated distance metrics have been developed, such as graphs, sequences, and logics.
    See \citet{ontanon2020overview} and references therein for many examples.
    Nevertheless, we do not have any theoretical grounding of this approach so far.
\end{remark}

\paragraph{Related work.}
We review several variants of similarity models used in existing literature.

\emph{(A) Reliable similarity.}
This line of work regards two data as similar if the associated labels are the same.
Our study and \citet{bellet2012similarity} belong to this category.
\citet{zhang2007value} proposed a method to decompose a model predicting pairwise labels into pointwise classifiers and analyzed the consistency of the model parameters.
\citet{hsu2018multi} have recently extended to the multi-class setup without theoretical justification yet.
In parallel, other research solved classification with pairwise supervision
by minimizing unbiased classification risk estimators~\citep{bao2018classification,shimada2019classification,cui2020classification}.
Their approaches are blessed with generalization error bounds,
while their performance deteriorates when the class-prior probability is close to uniform.
Note that even the reliable similarity can handle mild noise in pairwise supervision (Remark~\ref{remark:hard_similarity}).
Recently, \citet{tosh2021contrastive} revealed that pairwise supervision is sufficient to recover topic distributions under certain topic modeling assumptions.

\emph{(B) Noisy similarity.}
In this category, it is assumed that pairwise supervision aligns to the classes potentially with explicit noise.
For example, negative samples in contrastive learning are usually drawn from the marginal distribution, hence they could be false negatives.
Recently, \citet{chuang2020debiased} used techniques of unbiased risk estimators to improve the quality of negatives.
Further, contrastive learning often assumes that similar pairs share the same latent category, which can be different from downstream supervised classes.
Since contrastive learning is usually unsupervised, the supervised classes could be a subset or coarse-grained set of latent categories~\citep{saunshi2019theoretical}.
Other research modeled annotation errors in pairwise supervision~\citep{wu2020multi,dan2020learning}.

\emph{(C) Relaxation of positive-definite kernels.}
\citet{balcan2008theory} introduced $(\varepsilon,\gamma,\tau)$-good similarity to relax positive-definiteness of kernel functions, which supposes that a good similarity function is useful for downstream linear classification.
Much research in this framework has been interested in classification given features based on this weak similarity and derived classification error bounds~\citep{balcan2008theory,wang2009theory,kar2011similarity}.
Note that \citet{bellet2012similarity} assume reliable similarity as supervision and trains a similarity model while they train the model based on $(\varepsilon,\gamma,\tau)$-good similarity.
\section{Problem setup}
\label{sec:setup}
Let $\mathcal{X} \subseteq \R^{d}$ be a $d$-dimensional pattern space,
$\mathcal{Y}=\{\pm 1\}$ be the label space,
and $p(\x,y)$ be the density of an underlying distribution over $\mathcal{X} \times \mathcal{Y}$.
Denote the positive (negative, resp.) class prior by $\pi_+ \defeq p(y=+1)$ ($\pi_- \defeq p(y=-1)$, resp.).
Let $\sign(\alpha) = 1$ for $\alpha > 0$ and $-1$ otherwise.
$O_p(\cdot)$ denotes the order in probability.

\paragraph{Binary classification.}
The goal of binary classification is to classify unseen patterns into two classes.
It can be formulated as a problem to find a classifier $h:\mathcal{X} \rightarrow \mathcal{Y}$
that minimizes
\begin{equation}
\label{eq:classification_error}
\begin{split}
    & R_\PN (h) \defeq \E_{(X,Y)\sim p(\x, y)}
    \left[\indicator \{h(X) \neq Y \} \right],
\end{split}
\end{equation}
where $\indicator\{\cdot \}$ is the indicator function and
$\E_{(X,Y)\sim p(\x, y)}[\cdot]$ denotes the expectation with respect to $p(\x, y)$.
Typically,
we specify a hypothesis class $\mathcal{H}$ beforehand
and find a minimizer $h^*$ of $R_\PN$ in it:
$h^* \in \argmin_{h \in \mathcal{H}} R_\PN(h)$.
The empirical mean of $R_\PN$ is computed with finite samples.

\paragraph{Similarity learning.}
There are a variety of formulations of similarity learning such as (i) predicting whether a pair of patterns belong to the same class~\citep{zhang2007value,bellet2012similarity,hsu2018multi}, (ii) learning a metric that regards a similar pair of patterns closer~\citep{bilenko2004integrating,davis2007information,niu2014information,vogel2018probabilistic}, and (iii) learning a metric/representation that represents a similar pair more closer than background samples~\citep{wang2009theory,kar2011similarity,saunshi2019theoretical}.
Specifically,
we focus on the formulation (i) in the binary setup, which has a direct connection to classification (see Section~\ref{sec:classifier_from_pairwise}).
Hereafter, we suppose that a pair of $(\x,y)$ and $(\x',y')$ is independent of each other.
Let $\eta_{\pm 1}(\x) \defeq p(Y = \pm 1|X = \x)$.
Assume that $X=\x$ and $X'=\x'$ are observed first and pairwise supervision $T$ is drawn from
\begin{align*}
    & p(T\!=\!YY'|\x,\x') \\
    & =
    \begin{cases}
        \eta_{+1}(\x)\eta_{+1}(\x') + \eta_{-1}(\x)\eta_{-1}(\x') & \text{if $YY' = +1$,} \\
        \eta_{+1}(\x)\eta_{-1}(\x') + \eta_{-1}(\x)\eta_{+1}(\x') & \text{if $YY' = -1$.}
    \end{cases}
\end{align*}
The product $YY'$ indicates whether $Y$ and $Y'$ are the same/similar ($+1$) or not/dissimilar ($-1$).
Then, we are interested in the minimizer of the following classification error
\begin{equation}
\label{eq:similarity_learning_error}
    R_\SD (h) \defeq \hspace{-10pt} \E_{\substack{X, X' \sim p(\x) \\ T \sim p(T = YY'|\x, \x')}} \hspace{-10pt}
    \left[\indicator \{h(X) \cdot h(X') \neq T \} \right].
\end{equation}
Here, the model $h(\x) \cdot h(\x')$ is regarded as a similarity model so that we predict label agreement.
We call $R_\PN$ the \emph{pointwise classification error}
and $R_\SD$ the \emph{pairwise classification error}.
The empirical mean of $R_\SD$ is computed with a finite number of triplets $(\x, \x', yy')$.
We will discuss several benefits of the formulation \eqref{eq:similarity_learning_error} in Section~\ref{sec:relation_existing}.

\begin{remark}[Similarity as features]
    Similarity-based features are often used in domains where Euclidean features are unavailable~\citep{chen2009similarity,wang2009theory}.
    Under such a case, similarity-based features may be treated as $\x$ instead: given a number of ``landmark'' points $\{\z_1, \dots, \z_l\}$, a similarity function $K$ defines similarity-based features $[K(\x, \z_1), \dots, K(\x, \z_l)]^\top$ for an input $\x$.
    Our formulation assumes that $\x$ is available for simplicity but can be replaced with similarity-based features.
\end{remark}

\begin{remark}[$T$ is not a hard similarity label]
    \label{remark:hard_similarity}
    Even if $Y = Y' = +1$ (similar) with high probability, we could observe $T = -1$ (dissimilar) with some probability.
    Assume $\eta_{+1}(\x), \eta_{+1}(\x') \in (\tfrac{1}{2}, 1)$.
    Then, the flipping rate $p(T=-1|\x, \x')$ lies in $(0, \frac{1}{2})$.
    This means that observed pairwise supervision could be flipped stochastically under our similarity model.
    We expect that this generality is useful to handle annotation noise in pairwise supervision.
\end{remark}
\section{Learning a binary classifier with pairwise supervision}
\label{sec:classifier_from_pairwise}

We draw a connection between the specific formulation of similarity learning \eqref{eq:similarity_learning_error} and binary classification (Theorem~\ref{theo:relationship}).
This linkage enables us to train a pointwise binary classifier with pairwise supervision (Section~\ref{sec:method}).
All proofs hereafter are deferred to Appendix~\ref{sec:proof}.

\subsection{Connection between similarity learning and classification}
\label{sec:clustering_error}
We first introduce a performance metric for binary classification called the \emph{clustering error} that quantifies the discriminative power of a classifier up to label permutation:\footnote{
    $1 - R_\CE$ is known as clustering accuracy~\citep{fahad2014survey}.
    The number of clusters is confined to two for our purpose.
}
\begin{equation}
R_\CE (h) \defeq \min \{ R_\PN(h), R_\PN(-h) \}.
\end{equation}
Here, $R_\CE$ is used as an evaluator of binary classifiers,
though usually used for the evaluation of clustering methods~\citep{fahad2014survey}.
The clustering error differs from $R_\PN$ in that it dismisses the difference between $+h$ and $-h$,
yet a binary decision boundary is still evaluated properly.
The clustering error $R_\CE$ can be tied to the pairwise classification error $R_\SD$ as follows, which is our primary result.

\begin{theorem}
\label{theo:relationship}
For any classifier $h: \mathcal{X} \rightarrow \mathcal{Y}$, $0 \le R_\SD(h) \le \frac{1}{2}$, and
\begin{equation}
\label{eq:relationship}
    R_\CE(h) = \frac{1}{2} - \frac{\sqrt{1-2R_\SD(h)}}{2}.
\end{equation}
\end{theorem}
An immediate corollary is the monotonic relationship $R_\CE(h_1) < R_\CE(h_2) \iff R_\SD(h_1) < R_\SD(h_2)$ for any $h_1$ and $h_2$.
Hence, the minimization of $R_\SD$ amounts to the minimization of $R_\CE$,
constituting a decision boundary.
That is, \emph{similarity learning can essentially discover a binary decision boundary}.
While similarity learning has previously been connected to downstream classification via intermediate feature spaces~\citep{bellet2012similarity,saunshi2019theoretical,nozawa2019pac},
our result is the first to explicate that similarity learning is directly related to constructing a decision boundary.

\paragraph{Surrogate risk minimization.}
Here,
we discuss surrogate losses for similarity learning.
We define a hypothesis class by
$\mathcal{H} = \left\{ \sign \circ f \mid f \in \mathcal{F} \right\}$,
where $\mathcal{F} \subseteq \R^\mathcal{X}$ is a specified class of prediction functions
and $\sign \circ f(\cdot) \defeq \sign(f(\cdot))$.
Theorem~\ref{theo:relationship} suggests that we may minimize $R_\CE$ by minimizing $R_\SD$ instead.
As in the standard binary classification case,
the indicator function appearing in $R_\SD$ is replaced with a surrogate loss $\ell: \R \times \mathcal{Y} \to \R_{\geq 0}$
since it is intractable to minimize a discrete objective~\citep{bartlett2006convexity}.
Eventually,
the pairwise surrogate risk
\begin{equation}
\label{eq:surrogated_similarity_risk}
    R_{\SD}^\ell(f)
    \defeq \E_{X, X', T} \left[ \ell(f(X)f(X'), T) \right]
\end{equation}
is minimized.
If $\ell$ is \emph{classification-calibrated}~\citep{bartlett2006convexity}, the minimization of $R_\SD^\ell$ is expected to lead to minimizing $R_\SD$ as well.\footnote{
    If a surrogate loss $\ell$ is classification-calibrated, the minimization of the surrogate classification risk leads to minimizing the target classification error $R_\PN$.
    The precise definition can be found in \citet{bartlett2006convexity}.
    Typical loss functions such as the logistic and hinge losses are classification-calibrated.
}
This will be justified by Lemma~\ref{lemma:surrogate_calibration} in Section~\ref{sec:analysis}.

As we will discuss in Section~\ref{sec:relation_existing}, the formulation \eqref{eq:surrogated_similarity_risk} can be related to several existing formulations in similarity learning in terms of the surrogate loss.

\subsection{Determination of correct sign of classifiers}
\label{sec:class_assignment}
In Section~\ref{sec:clustering_error}, we observed that similarity learning can draw a decision boundary up to label permutation.
For a given hypothesis $h$,
we are now interested in its sign, i.e., $+h$ or $-h$, leading to a smaller pointwise classification error.
We refer to this step as class assignment.
The optimal class assignment is denoted by
$s^* \defeq \argmin_{s \in \{\pm 1\}} R_\PN(s \cdot h)$.
We can consider two scenarios.
Under both, class assignment is much cheaper in supervision than training the post-hoc linear separators.

\paragraph{Class assignment with pointwise supervision.}
If pointwise supervision is available,
we can determine the class assignment by minimizing the pointwise classification error $R_\PN$ computed with the additional data.
This procedure admits the exponentially small sample complexity~\citep{zhang2007value}.

\paragraph{Class assignment without pointwise supervision.}
Here, we further ask
if it is possible to obtain the correct class assignment \emph{without} any class labels.
Surprisingly, we find that this is possible if the positive and negative proportions are not equal and we know \emph{which class is the majority}.
Based on the equivalent expression of $R_\PN$~\citep{shimada2019classification},
this finding is formally stated in the following theorem.
\begin{theorem}
\label{theo:class_assignment}
Assume that the class prior $\pi_+ \ne \tfrac{1}{2}$.
Then, the optimal class assignment $s^*$ can be represented as
$s^* = \sign (2 \pi_+ - 1) \cdot \sign (1 - 2 Q(h))$,
where
\begin{equation*}
    Q(h) \defeq \E_{X, X', T} \left[ \frac{\indicator\left\{h(X) \neq T \right\} + \indicator\left\{h(X') \neq T \right\}}{2} \right].
\end{equation*}
\end{theorem}
We approximate $Q$ with a finite number of pairs.
As we will see in Lemma~\ref{lemma:est_class_assingment} in Section~\ref{sec:analysis},
the class assignment error is exponentially small in the number of pairs.

\begin{remark}[Necessity of $Q(h)$]
If we know which class is the majority,
class assignment may look possible at a glance
by simply looking at the average of $h(\x)$ with unlabeled validation data,
instead of Theorem~\ref{theo:class_assignment}.
Unfortunately, this does not always succeed even asymptotically
(discussed in Appendix~\ref{subsec:class-assignment}).
\end{remark}

\subsection{Learning a binary classifier with only pairwise supervision is possible}
\label{sec:method}
As a by-product of Theorems~\ref{theo:relationship} and \ref{theo:class_assignment},
the following two-stage method can train a pointwise classifier with only pairwise supervision.
Assume that
the class prior is not $\frac{1}{2}$ and the majority class is known.
Let $\mathfrak{D}_\textrm{train} \defeq \{(\x_i, \x_i', \tau_i) \}_{i=1}^{n_\SD}$ be a training set,
where $\tau_i \defeq y_i  y'_i$ and $(\x_i,y_i)$ and $(\x_i',y_i')$ are i.i.d.~samples following $p(\x,y)$.
We randomly divide $n_\SD$ pairs in $\mathfrak{D}_\mathrm{train}$ into two sets $\mathfrak{D}_1$ and $\mathfrak{D}_2$,
where $|\mathfrak{D}_1| = m_1$ and  $|\mathfrak{D}_2| = m_2$ satisfying $m_1 + m_2 = n_\SD$.\footnote{
    The independent two sets are necessary otherwise errors of Steps 1 and 2 correlate, which leads to overfitting.
    Technically, they are required because Theorem~\ref{theo:est_proposed} relies on the union bound.
}

In Step~1, we obtain a minimizer of the empirical pairwise classification risk with $\mathfrak{D}_1$:
\begin{equation}
\label{eq:f_hat}
    \hat{f} \defeq \argmin_{f \in \mathcal{F}} \hat{R}^\ell_\SD(f),
\end{equation}
where $\hat{R}^\ell_\SD$ is the sample mean of $R^\ell_\SD$ with $\mathfrak{D}_1$.
In Step~2, we assign classes with $\sign \circ \hat{f}$ and $\mathfrak{D}_2$:
\begin{equation}
\label{eq:s_hat}
    \hat{s} \defeq \sign (2 \pi_+ - 1) \cdot \sign (1 - 2 \hat{Q}(\sign \circ \hat{f})),
\end{equation}
where $\hat Q$ is the sample mean of $Q$ with $\mathfrak{D}_2$.
After all, $\hat{s} \cdot \sign \circ \hat{f}$ is a desideratum.
If class assignment is not necessary and just separating test patterns into two disjoint groups is the goal, we may simply set $m_1=n_\SD$ and omit Step~2 of finding $\widehat{s}$.

\begin{remark}[Case of $\pi = \frac{1}{2}$]
    With only pairwise supervision,
    class assignment is hopeless because both classes are essentially symmetric,
    while it is still possible to draw a decision boundary.
    Class assignment with pointwise supervision is still possible.
\end{remark}

\subsection{Benefits of our formulation over existing similarity learning}
\label{sec:relation_existing}

\begin{table*}[t]
    \centering
    \caption{
        Comparison of closely related methods to train classifiers with pairwise supervision.
        They assume the availability of reliable similarity (see Section~\ref{sec:introduction}).
        The column ``$\pi_+ = \frac{1}{2}$'' shows whether the formulation is valid under $\pi_+ = \frac{1}{2}$.
        In \emph{sample complexity}, $m$ denotes the number of paired data in Step~1,
        and either paired or pointwise data in Step~2.
        The sample complexity analysis of Step~1 is with respect to either pointwise classification or clustering error.
        To make the comparison proper, we assume that the hinge loss is used and eventually $\psi$-transform is $\psi(u) = u$.
        This is detailed in Section~\ref{sec:analysis} (Discussion).
    }
    \footnotesize
    \vspace{2pt}
    \label{tab:comparison}
    \renewcommand{\arraystretch}{1.2}
    \begin{tabular}{lllll}
        \toprule
        {} & {} & \multicolumn{2}{c}{Sample complexity of} & {} \\
        \cmidrule(lr){3-4}
        {} & \makecell{$\pi_+ = \frac{1}{2}$} & \makecell{Similarity learning (Step 1)} & \makecell{Post-process (Step 2)} & Comment \\
        \midrule
        \textbf{CIPS} (Ours) & \cmark & \makecell[l]{$O_p(m^{-\frac{1}{4}})$ \\ (Lemma~\ref{lemma:est_clustering_error_minimization} in \S\ref{sec:analysis})} & \makecell[l]{$O_p(e^{-m})$ \\ (Lemma~\ref{lemma:est_class_assingment} in \S\ref{sec:analysis})} & \makecell[l]{Step 2 is class assignment.} \\
        \midrule
        \makecell[l]{OVPC \\ \citep{zhang2007value}} & \cmark & (N/A) & $O_p(e^{-m})$ & \makecell[l]{Step 2 is class assignment. Step~1 \\ was shown to be consistent \\ but complexity is not known.} \\
        \makecell[l]{SLLC \\ \citep{bellet2012similarity}} & \cmark & $O_p(m^{-\frac{1}{4}})$ & $O_p(m^{-\frac{1}{2}})$ & \makecell[l]{Step 2 is SVM training.} \\
        \makecell[l]{MCL \\ \citep{hsu2018multi}} & \cmark & (N/A) & (N/A) & \makecell[l]{Inner product of classifiers is fitted in \\ Step 1. Sample complexities have yet \\ to be known.} \\
        \makecell[l]{SD \\ \citep{shimada2019classification}} & -- & $O_p(m^{-\frac{1}{2}})$ & (unnecessary) & Step 1 trains classifiers directly. \\
        \bottomrule
    \end{tabular}
\end{table*}

We reiterate that similarity learning in our formulation directly elicits a boundary without the post-process in contrast with \citet{bellet2012similarity}---their method needs to train a classifier built on top of the learned similarity metric in the post-process, which incurs additional sample complexity $O_p(m^{-1/2})$.
Table~\ref{tab:comparison} provides an overview of the comparison with related work.
We remark that the sample complexity of SLLC is transformed into the complexity in terms of paired data (Step~1) from the original complexity in pointwise data~\citep[Theorem~3]{bellet2012similarity}.\footnote{
    Given $m$ pointwise data, $O(m^2)$ pairs can be generated
    and thereby the sample complexity is transformed.
    Strictly speaking, the generated $O(m^2)$ points are not independent of each other.
    Nevertheless, the convergence rate would remain the same by using the error bound with interdependent data~\citep{usunier2005generalization}.
}
In addition, while our Step~1 is worse than SD, our formulation is valid even when $\pi_+ = \frac{1}{2}$ with pointwise supervision.
Subsequently, we discuss the other perspectives of our formulation.

\paragraph{Generalization in terms of surrogate losses.}
Several existing formulations can be related to our formulation \eqref{eq:surrogated_similarity_risk}.
Kernel alignment~\citep{cristianini2002kernel} learns a kernel $K$ approximating a similarity matrix $K^*$ of labels
by maximizing the cosine similarity $\frac{\langle K, K^* \rangle}{\sqrt{\|K\| \cdot \|K^*\|}}$,
where $\langle K, K^* \rangle$ is the Frobenius inner product of the Gram matrices.
If the product $f(\x) \cdot f(\x')$ is used as a kernel,
kernel alignment is equivalent (up to the normalization factor $\sqrt{\|K\| \|K^*\|}$) to minimizing Eq.~\eqref{eq:surrogated_similarity_risk} with the linear loss $\ell_\mathrm{lin}(z, t) \defeq -zt$.
On the other hand, metric learning based on ($\varepsilon$,$\gamma$,$\tau$)-good similarity~\citep{balcan2008theory} regards a similarity function inducing a good linear separator as a good similarity.
Here, the linear separability is defined via the hinge loss $\ell_\mathrm{hinge}(z, t) \defeq [1-zt]_+$.
\citet{bellet2012similarity} formulated learning a bilinear similarity $\x^\top A \x'$ by minimizing the hinge loss,
which is equivalent to the minimization of Eq.~\eqref{eq:surrogated_similarity_risk} with $\ell_\mathrm{hinge}$ and the choice $A = \w\w^\top$ such that $f(\x) = \w^\top\x$.
In other words, we posit the rank-$1$ similarity model in order to have Theorem~\ref{theo:relationship}.
In addition to these examples, the InfoNCE loss used in recent contrastive learning~\citep{oord2018representation,logeswaran2018efficient,saunshi2019theoretical} can be regarded as the (multi-sample counterpart of) logistic loss $\ell_\mathrm{log}(z,t) \defeq \log(1+e^{-zt})$.

Thanks to this generalization, subsequent analysis systematically connects these existing formulations to downstream classification under the model assumption.

\paragraph{Explicit relation to classification.}
\citet{hsu2018multi} formulated similarity learning in a slightly different way,
as maximum likelihood estimation of the pairwise label $S_\tau \defeq \frac{\tau + 1}{2}$:\footnote{The multi-class formulation in \citet{hsu2018multi} was simplified in binary classification here for comparison.}
\begin{align}
    \begin{aligned}
        \min_{f \in \mathcal{F}} \frac{1}{m_1} & \sum_{(\x,\x',\tau) \in \mathfrak{D}_1} \hspace{-10pt} -S_{\tau}\log(\tilde q(f(\x), f(\x')) \\
        & - (1-S_\tau)\log(1 - \tilde q(f(\x), f(\x'))),
    \end{aligned}
    \label{eq:lps_model}
\end{align}
where $\tilde q(z, z') \defeq \left[\begin{smallmatrix}q(z)\\1-q(z)\end{smallmatrix}\right]^\top\left[\begin{smallmatrix}q(z')\\1-q(z')\end{smallmatrix}\right]$ is the inner product of two binary probability vectors,
and $q(z) \defeq (1 + \exp(-z))^{-1}$ denotes the (inverse) logit link.
On the other hand, our formulation \eqref{eq:f_hat} with the logistic loss $\ell_\mathrm{log}(z,t) = -S_t\log(q(z)) -(1-S_t)\log(1-q(z))$ is
\begin{align}
    \begin{aligned}
        \min_{f \in \mathcal{F}} \frac{1}{m_1} & \sum_{(\x,\x',\tau) \in \mathfrak{D}_1} \hspace{-10pt} -S_\tau\log(q(f(\x) \cdot f(\x'))) \\
        & - (1-S_\tau)\log(1-q(f(\x) \cdot f(\x'))).
    \end{aligned}
    \label{eq:ips_model}
\end{align}
In the formulation \eqref{eq:lps_model}, similarity is defined by the inner product of class probabilities,
while it is defined by the inner product of $f$ in the formulation \eqref{eq:ips_model}.
The latter definition is often called the \emph{inner product similarity} (IPS) model~\citep{okuno2020hyperlink}.\footnote{
    The IPS model originally defined similarity between two vector data representations,
    hence is called \emph{inner} product similarity.
    Yet, the IPS model is applied on one-dimensional prediction $f(\x)$ in our context.
    The IPS model has been used in several domains~\citep{tang2015line,logeswaran2018efficient,saunshi2019theoretical,okuno2020hyperlink}.
}
While both are valid similarity learning methods,
the IPS model \eqref{eq:ips_model} has several benefits: one can choose arbitrary loss functions,%
\footnote{
    The formulation \eqref{eq:lps_model} can be extended from maximum likelihood estimation by using an arbitrary proper scoring rules~\citep{gneiting2007strictly}, but non-proper losses such as the hinge loss cannot be used.
}
and besides, the pairwise classification risk minimization~\eqref{eq:f_hat} admits an excess risk bound (Lemma~\ref{lemma:surrogate_calibration} in Section~\ref{sec:analysis}).
For this reason, we call our formulation \emph{CIPS (Classifier with Inner Product Similarity)} from now on.

\section{Excess risk and sample complexity analysis}
\label{sec:analysis}

In this section, we provide the missing sample complexity analyses of CIPS in Table~\ref{tab:comparison}.
In addition, the excess risk is obtained to claim that CIPS does solve binary classification.

Let $\hat{f}$ and $\hat{s}$ be the solutions of Eqs.~\eqref{eq:f_hat} and \eqref{eq:s_hat}, respectively.
The target excess risk for similarity learning is denoted by
\begin{equation*}
    \Er_\PN(\hat{s} \cdot \sign \circ \hat{f}) \defeq  R_\PN(\hat{s} \cdot \sign \circ \hat{f} ) - R_\PN^*,
\end{equation*}
where $R_\PN^* \defeq \inf\limits_{f} R_\PN(\sign \circ f)$,
and $\inf\limits_{f}$ indicates the infimum over all measurable functions.
In addition, we introduce notation for the other excess risks:
\begin{align*}
    \Er_\CE(\sign \circ f) &\defeq R_\CE(\sign \circ f) - R_\CE^*, \\
    \Er_\SD(\sign \circ f) &\defeq R_\SD(\sign \circ f) - R_\SD^*, \\
    \Er_\SD^\ell(f) &\defeq R_\SD^\ell(f) - R_\SD^{\ell,*},
\end{align*}
where $R_\CE^* \defeq \inf\limits_f R_\CE(\sign \circ f)$.
$R_\SD^*$ and $R_\SD^{\ell,*}$ are defined as the infima over all measurable functions similarly.
To derive the excess risk bound on $\Er_\PN(\hat{s} \cdot \sign \circ \hat{f})$,
we need to handle errors of clustering error minimization and class assignment independently,
which will be shown in Lemmas~\ref{lemma:est_clustering_error_minimization} and \ref{lemma:est_class_assingment}, respectively.
An important insight to combine two errors is that
if the class assignment is successful, $\Er_\PN(\hat{s} \cdot \sign \circ \hat{f})$ is equivalent to the excess risk of clustering error minimization.
That is to say,
\begin{equation}
\label{eq:successful_class_assignment}
\begin{split}
    & \hat{s} = \argmin_{s \in \{\pm 1\}}R_\PN(s \cdot \sign \circ \hat{f}) \\
    & \implies \Er_\PN(\hat{s} \cdot \sign \circ \hat{f}) = \Er_\CE(\sign \circ \hat{f}).
\end{split}
\end{equation}
In order to bound $\Er_\CE(\sign \circ \hat{f})$,
we use the Rademacher complexity~\citep{bartlett2002rademacher}
specifically defined on the class $\{ (\x, \x') \mapsto f(\x) \cdot f(\x') \mid f \in \mathcal{F} \}$
\begin{equation*}
    \mathfrak{R}_m(\mathcal{F})
    \defeq \E_{X_i, X_i'}
    \left[ \sup_{f \in \mathcal{F}} \frac{1}{m} \sum_{i=1}^{m} \sigma_i f(X_i) \cdot f(X'_i) \right],
\end{equation*}
where $\{\sigma_i\}_{i=1}^m$ are the Rademacher variables.
Before obtaining an excess risk bound of $R_\CE$, we need to bridge the excess risk $\Er_\SD$ and the surrogate $\Er_\SD^\ell$.
\begin{lemma}
\label{lemma:surrogate_calibration}
If a loss $\ell$ is classification-calibrated~\citep{bartlett2006convexity},
then there exists a convex, non-decreasing, and invertible $\psi: [0, 1] \to [0, +\infty)$
such that for any sequence $(u_i)$ in $[0, 1]$,
\begin{align*}
    \psi(u_i) \to 0 \text{~~~if and only if~~~} u_i \to 0
\end{align*}
and for any measurable function $f$ and probability distribution on $\mathcal{X} \times \mathcal{Y}$,
\begin{align*}
    \psi(\Er_\SD(\sign \circ f)) \leq \Er_\SD^\ell(f).
\end{align*}
\end{lemma}
Although the similar result to Lemma~\ref{lemma:surrogate_calibration} has already been known for $R_\PN$~\citep[Theorem~1]{bartlett2006convexity},
the proof for $R_\SD$ requires special care to treat the product of prediction functions properly.

Then, the excess risk bound for $R_\CE$ is derived
based on Lemma~\ref{lemma:surrogate_calibration} and the uniform bound.
\begin{lemma}
\label{lemma:est_clustering_error_minimization}
Let $f^* \in \mathcal{F}$ be a minimizer of $R_\SD^\ell$,
and $\hat{f} \in \mathcal{F}$ be a minimizer of $\widehat{R}_\SD^\ell$ defined in Eq.~\eqref{eq:f_hat}.
Assume that $\ell(\cdot,\pm 1)$ is $\rho$-Lipschitz ($0 < \rho < \infty$),
and that $\| f \|_\infty \leq C_b$ for any $f \in \mathcal{F}$ for some $C_b$.
Let $C_\ell \defeq \sup_{t \in \{\pm1 \}} \ell(C_b^2, t)$.
For any $\delta > 0$, with probability at least $1-\delta$,
\begin{align*}
    & \Er_\CE(\sign \circ \hat f) \\
    & \le \sqrt{\frac{1}{2}\psi^{-1} \bigg(
        \Er_\SD^\ell(f^*) +
        4\rho \mathfrak{R}_{m_1}(\mathcal{F}) + \sqrt{\frac{2 C_\ell^2 \log\frac{2}{\delta} }{m_1}}
    \bigg)}.
\end{align*}
\end{lemma}

Next, the class assignment error probability using pairwise supervision is analyzed.
\begin{lemma}
\label{lemma:est_class_assingment}
Assume that $\pi_+ \neq \frac{1}{2}$.
Let $\hat{s}$ be the solution defined in Eq.~\eqref{eq:s_hat}.
Then, we have
\begin{align*}
    & \Pr \Big( \hat{s} \neq \argmin_{s \in \{\pm 1\}}R_\PN(s \cdot \sign \circ \hat{f}) \Big) \\
    & \leq \exp \Big( -\frac{m_2}{2} (2 \pi_+ - 1)^2 \big(2 R_\PN(\sign \circ \hat{f}) - 1 \big)^2 \Big).
\end{align*}
\end{lemma}
Several observations from Lemma~\ref{lemma:est_class_assingment} follow.
As $\pi_+ \to \frac{1}{2}$,
the upper bound becomes looser.
This comes from the fact that the estimation of the pointwise classification error with pairwise supervision becomes more difficult as $\pi_+ \to \frac{1}{2}$~\citep{shimada2019classification}.
Moreover,
the discriminability of function $\hat{f}$,
i.e., $R_\PN(\sign \circ \hat{f})$,
appears in the inequality and thus it is directly related to the error rate.
Intuitively,
if a given function classifies a large portion of data correctly,
the optimal sign can be identified easily.

Finally, an overall excess risk bound is derived
by combining Lemmas~\ref{lemma:est_clustering_error_minimization}, \ref{lemma:est_class_assingment}, and the fact \eqref{eq:successful_class_assignment}.
Let $\Er_\PN(h)$ denote the excess risk $R_\PN(h) - R_\PN^*$.
\begin{theorem}
\label{theo:est_proposed}
Suppose that we have $\pi_+ \neq \frac{1}{2}$.
Let $r \defeq \exp ( -\frac{m_2}{2} (2 \pi_+ - 1)^2 (2 R_\PN(\sign \circ \hat{f}) - 1 )^2 )$.
Under the same assumptions as Lemma~\ref{lemma:est_clustering_error_minimization},
for any $\delta > r$, with probability at least $1-\delta$,
\begin{align*}
    & \Er_\PN (\hat{s} \cdot \sign \circ \hat{f}) \\
    & \leq \sqrt{ \! \frac{1}{2}\psi^{-1}\bigg( \!
        \Er_\SD^\ell(f^*) +
        4\rho \mathfrak{R}_{m_1}(\mathcal{F}) \! + \! \sqrt{\frac{2 C_\ell^2 \log\tfrac{2}{\delta-r}}{m_1}}
    \! \bigg)}.
\end{align*}
\end{theorem}
In the proof of Theorem~\ref{theo:est_proposed}, the surrogate excess risk $\Er_\SD^\ell(\hat{f})$ is decomposed into the estimation error and the approximation error $\Er_\SD^\ell(f^*)$.
If $\mathfrak{R}_{m_1}(\mathcal{F}) = o(1)$,
the estimation error asymptotically vanishes and the upper bound approaches to the approximation error in probability.
Under this condition, similarity learning successfully minimizes our desideratum $\Er_\PN$, with a flexible enough $\mathcal{F}$ entailing the small approximation error.
For example,
linear-in-parameter model $\mathcal{F}=\left\{ f(\x) = \w^\top \boldsymbol{\phi}(\x) + b \right\}$ satisfies $\mathfrak{R}_{m_1}(\mathcal{F}) = O({m_1}^{-\frac{1}{2}})$
as shown in \citet[Lemma~5]{kuroki2019unsupervised},
where $\w \in \R^k$ and $b \in \R$ are weights and bias parameters and $\boldsymbol{\phi}: \R^d \rightarrow \R^k$ are mapping functions.
Note that our result is stronger than \citet{zhang2007value}
because they only provided the asymptotic convergence,
while Theorem~\ref{theo:est_proposed} provides a finite sample guarantee.

\paragraph{Discussion.}
Since class assignment admits the exponential decay of the error probability (Lemma~\ref{lemma:est_class_assingment}) under the moderate condition ($\pi_+ \ne \frac{1}{2}$),
we may set $m_2 \ll m_1$ in practice.
In contrast, our excess risk bound of clustering error minimization (Lemma~\ref{lemma:est_clustering_error_minimization}) is governed in part by $\psi$.
The explicit rate depends on specific choices of loss functions:
e.g., the hinge loss gives $\psi(u) = u$, and under the assumption $\mathfrak{R}_{m_1}(\mathcal{F}) = O({m_1}^{-\frac{1}{2}})$,
the explicit rate is $O_p({m_1}^{-\frac{1}{4}})$.%
\footnote{
    As another example, the logistic loss gives $\psi(u) = \Omega(u^2)$,
    entailing the explicit rate $O_p({m_1}^{-\frac{1}{8}})$ for the excess risk bound (Lemma~\ref{lemma:est_clustering_error_minimization}).
    For more examples of $\psi$, see \citet[Table~1]{steinwart2007compare}.
}
This rate is no slower than the pointwisely supervised case $O_p({m}^{-\frac{1}{2}})$
because $O({m}^2)$ pairwise supervision can be generated with $m$ pointwise labels.

Note again that CIPS assumes $\pi_+ \ne \frac{1}{2}$ only in class assignment (Step~2 \& Lemma~\ref{lemma:est_class_assingment}),
not in clustering error minimization (Step~1 \& Lemma~\ref{lemma:est_clustering_error_minimization}).
This is a subtle but notable difference from earlier similarity learning methods based on unbiased classification risk estimators,
which requires $\pi_+ \ne \frac{1}{2}$ even in risk minimization (see \citet{shimada2019classification}).

Our excess risk bound (Theorem~\ref{theo:est_proposed}) resembles transfer bounds among binary classification, class probability estimation (CPE), and bipartite ranking.
\citet{narasimhan2013relationship} reduced classification and CPE to ranking
and showed that the excess risks of both classification and CPE can be upper-bounded by that of ranking.
As can be seen in \citet{narasimhan2013relationship},
the excess risk of classification/CPE slows down to be $O(\lambda(m)^{-\frac{1}{2}})$ suppose that the excess risk of ranking is $\lambda(m)$.
The same decay is observed in Theorem~\ref{theo:est_proposed} as well, reducing classification to similarity learning.
This decay $O((\cdot)^{-\frac{1}{2}})$ can be regarded as a \emph{cost arising from problem reduction}.

\section{Experiments}
\label{sec:experiments}

\begin{figure*}[t]
\centering
\begin{minipage}{.60\textwidth}
  \centering
  \subfigure[$\langle\clubsuit\rangle$ $\pi_+ = 0.5$, various $m$ \label{fig:consistency_mnist}]{\includegraphics[height=150pt]{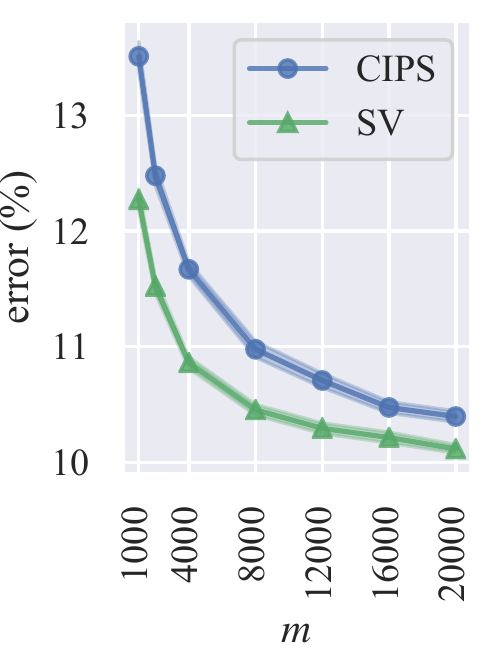}}
  \subfigure[$\langle\heartsuit\rangle$ various $\pi_+$, $m = 10000$ \label{fig:prior_mnist}]{\includegraphics[height=155pt]{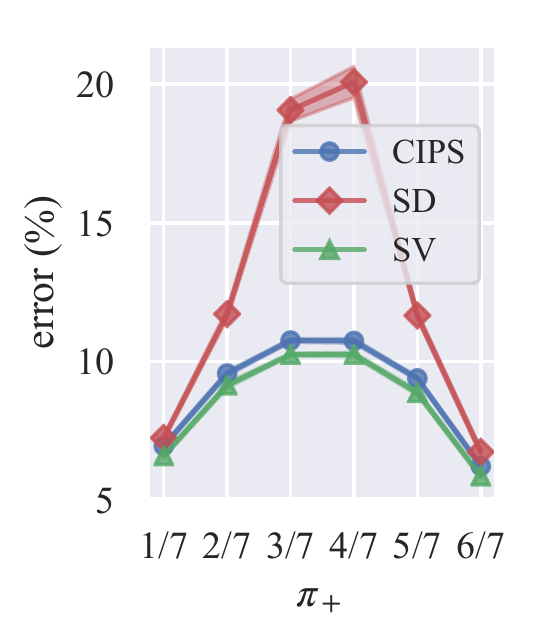}}
  \captionsetup{width=.92\linewidth}
  \captionof{figure}{
    (left) Mean clustering error and standard error (shaded areas) over $20$ trials on MNIST.
    (right) Mean clustering error and standard error (shaded areas) over $10$ trials on MNIST.
  }
\end{minipage} \hfill
\begin{minipage}{.38\textwidth}
  \centering
  \includegraphics[height=150pt]{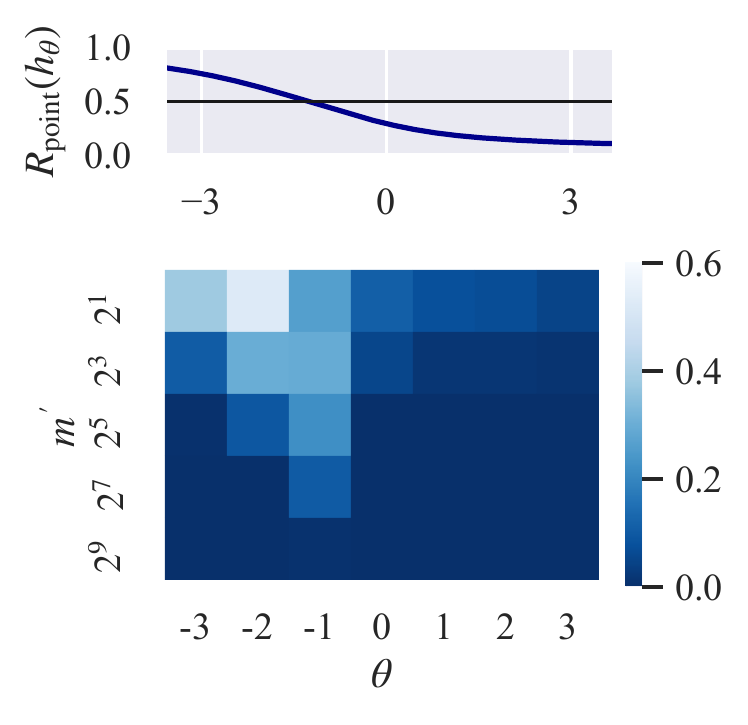}
  \captionsetup{width=.92\linewidth}
  \captionof{figure}{
    $\langle\spadesuit\rangle$ Classification error for each threshold classifier (upper) and
    the error probability of the proposed class assignment method over $\num{10000}$ trials (bottom) on the synthetic Gaussian dataset with $\pi_+ = 0.1$
  }
  \label{fig:class_assignment}
\end{minipage}%
\end{figure*}

This section shows simulation results to confirm our findings:
$\langle\clubsuit\rangle$ the sample complexity of the clustering error minimization via similarity learning (Lemma~\ref{lemma:est_clustering_error_minimization}),
$\langle\heartsuit\rangle$ the class-prior effect in similarity learning (Discussion in Section~\ref{sec:analysis}),
and
$\langle\spadesuit\rangle$ class assignment without pointwise supervision (Lemma~\ref{lemma:est_class_assingment}).
In addition, we compared with baselines using benchmark and real-world datasets (PubMed-Diabetes).
All experiments except PubMed-Diabetes were carried out with 3.60GHz Intel\textsuperscript{\tiny\textregistered} Core\textsuperscript{TM} i7-7700 CPU and GeForce GTX 1070.
Experiments on PubMed-Diabetes were carried out with 1.40GHz Intel\textsuperscript{\tiny\textregistered} Xeon Phi\textsuperscript{TM} 7250.
Full results are included in Appendix~\ref{sec:experiments-full}.
All simulation codes are available in the supplementary material.

\paragraph{Clustering error minimization on benchmark datasets.}
Tabular datasets from LIBSVM~\citep{chang2011libsvm} and UCI~\citep{Dua:2019} repositories
and MNIST dataset~\citep{lecun1998mnist} were used in benchmarks.
The labels of MNIST were binarized into even vs. odd digits.
Pairwise supervision was generated by random coupling of pointwise data in the original datasets.
We briefly introduce baselines below.
Constrained $k$-means clustering (CKM)~\citep{wagstaff2001constrained} and semi-supervised spectral clustering (SSP)~\citep{chen2012spectral} are semi-supervised clustering methods based on $k$-means~\citep{macqueen1967some} and spectral clustering~\citep{von2007tutorial}, respectively.
A method proposed by \citet{zhang2007value} (OVPC) and similar-dissimilar classification (SD)~\citep{shimada2019classification} are classification methods using pairwise supervision, which admit the generalization guarantee.
Meta-classification likelihood (MCL)~\citep{hsu2018multi} is an approach based on maximum likelihood estimation over pairwise labels.
For reference,
$k$-means clustering (KM) and supervised learning (SV) were compared.
For classification methods that require model specification (i.e., CIPS, SD, MCL, OVPC, and SV),
a linear model $f(\x)=\w^\top \x + b$ was used.
For CIPS, SD, and SV, we used the logistic loss, which is classification-calibrated.
The rest of implementation details is deferred to Appendix~\ref{sec:experiments-full}.

$\langle\clubsuit\rangle$ First,
in order to verify the sample complexity behavior in Lemma~\ref{lemma:est_clustering_error_minimization},
classifiers were trained with MNIST.
The number of pairwise data $m$ was set to each of $\{ \num{1000}, \num{2000}, \num{4000}, \num{8000}, \num{12000}, \num{16000}, \num{20000} \}$.
Figure~\ref{fig:consistency_mnist} presents the performances of CIPS and SV.
This demonstrates that the clustering error of CIPS constantly decreases as $m$ grows,
which is consistent with Lemma~\ref{lemma:est_clustering_error_minimization}.
Moreover,
CIPS performed more efficiently than expected in terms of sample complexity---as we discussed in Section~\ref{sec:analysis},
we expect that CIPS with $O(m^2)$ pairs performs comparably to SV with $m$ data points.

$\langle\heartsuit\rangle$ Next,
to see the effect of the class prior,
we compared CIPS, SD, and SV with various class priors.
In this experiment,
train and test data were generated from MNIST under the controlled class prior $\pi_+$,
where $\pi_+$ was set to each of $\{\frac{1}{7}, \ldots, \frac{6}{7} \}$.
For each trial,
$\num{10000}$ pairs were randomly subsampled from MNIST for training and the performance was evaluated with another $\num{10000}$ labeled examples.
The average clustering errors and standard errors over ten trials are plotted in Figure~\ref{fig:prior_mnist}.
This result indicates that CIPS is less affected compared with SD.

Finally, we show the benchmark performances of each method on the tabular datasets in Table~\ref{table:tabular_classification},
where each cell contains the average clustering error and the standard error over $\num{20}$ trials.
For each trial,
we subsampled $m \in \{100, 1000\}$ pairs for training data
and $\num{1000}$ pointwise examples for evaluation.
This result demonstrates CIPS performs better with large enough samples than most of the baselines and comparably to MCL.
The performance difference between CIPS and clustering methods
implies that larger samples do improve the downstream performance of CIPS thanks to its generalization guarantee (Theorem~\ref{theo:est_proposed}).

\begin{table*}[t]
\centering
\caption{
Mean clustering error and standard error on different benchmark datasets over $20$ trials.
Bold numbers indicate outperforming methods (excluding SV):
among each configuration, the best one is chosen first,
and then the comparable ones are chosen by one-sided t-test with the significance level $5\%$.
}
\label{table:tabular_classification}
\scalebox{0.9}{
\begin{tabular}{cccccccccc}
\toprule
        dataset & \multirow{2}{*}{$m$ }&         \multirow{2}{*}{CIPS (Ours)} &         \multirow{2}{*}{MCL} &   \multirow{2}{*}{SD} &  \multirow{2}{*}{OVPC} & \multirow{2}{*}{SSP} & \multirow{2}{*}{CKM} & \multirow{2}{*}{KM} &  \multirow{2}{*}{(SV)}\\
        (dim., $\pi_+$) & \\
\midrule
adult & 100  &  39.8 (1.6) &  38.4 (2.1) &  30.8 (0.9) &  45.0 (0.9) &  \bf 24.7 (0.3) &  28.9 (0.8) &  \bf 24.9 (0.5) & 21.9 (0.4)\\
    (123, 0.24)     & 1000 & \bf  17.6 (0.3) & \bf  17.2 (0.3) &  20.5 (0.3) &  45.5 (0.7) &  24.2 (0.3) &  27.9 (0.4) &  27.9 (0.5) & 15.9 (0.3)\\
\midrule
codrna & 100  & \bf  24.7 (1.8) &  32.3 (1.4) & \bf  28.0 (1.3) &  32.0 (2.0) &  45.5 (1.5) &  46.7 (0.6) &  42.5 (1.0) & 11.0 (0.6) \\
   (8, 0.33)      & 1000 &  \bf  6.3 (0.2) &  \bf  6.5 (0.2) &   8.8 (0.4) &  28.3 (2.0) &  44.8 (1.6) &  46.1 (0.4) &  45.4 (0.6) & 6.3 (0.2) \\
\midrule
ijcnn1 & 100  &  16.6 (2.3) &  24.9 (2.9) & \bf  10.7 (0.3) &  41.1 (1.1) &  31.6 (2.0) &  40.0 (1.3) &  31.9 (2.4) & 9.1 (0.2) \\
   (22, 0.10)      & 1000 &  \bf  7.7 (0.2) &  \bf  7.9 (0.2) & \bf   8.1 (0.2) &  42.0 (1.4) &  34.9 (1.7) &  45.9 (0.8) &  43.4 (0.7) & 7.6 (0.2)\\
\midrule
phishing & 100  & \bf  12.7 (2.3) &  \bf 12.8 (2.3) &  34.6 (1.8) &  41.7 (1.0) &  46.6 (0.5) &  24.4 (3.4) &  47.0 (0.5) & 7.6 (0.2)\\

  (44, 0.68)       & 1000 & \bf   6.5 (0.2) &  \bf  6.3 (0.2) &  22.0 (1.0) &  43.8 (1.1) &  45.5 (0.5) &  15.2 (2.7) &  46.4 (0.5) & 6.3 (0.2)\\
\midrule
w8a & 100  &  31.5 (1.9) &  31.4 (2.1) &  11.8 (0.3) &  39.7 (1.4) & \bf   5.3 (1.2) & \bf   6.8 (1.9) & \bf   5.5 (1.3) & 10.3 (0.4)\\
    (300, 0.03)     & 1000 &   2.6 (0.2) &  \bf  2.2 (0.1) &   2.6 (0.2) &  43.1 (0.8) &   3.0 (0.1) &   8.9 (2.6) &   3.7 (0.5) & 2.0 (0.1) \\
\bottomrule
\end{tabular}
}
\end{table*}

\begin{table}[t]
  \centering
  \caption{
    Mean clustering error and standard error on Pubmed-Diabetes dataset over 20 trials.
    Bold numbers indicate outperforming method (excluding SV):
    chosen by the one-sided t-test in the same way as Table~\ref{table:tabular_classification}.
  }
  \vspace{2pt}
  \footnotesize
  \label{table:pubmed}
  \begin{tabular}{lllll}
    \toprule
    CIPS (Ours) & MCL & DML & (SV) \\
    \midrule
    \textbf{86.9 (0.4)} &  \textbf{86.6 (0.4)} &  85.1 (0.2) &  94.7 (0.1) \\
    \bottomrule
  \end{tabular}
\end{table}

\paragraph{Class assignment on synthetic dataset.}
The performance of the proposed class assignment method was empirically investigated on synthetic dataset.
The class-conditional distributions with the standard Gaussian distributions
were used as the underlying distribution:
$p(\x|y=+1) = \mathcal{N}(\x|\mu_+, \sigma_+)$ and
$p(\x|y=-1) = \mathcal{N}(\x|\mu_-, \sigma_-)$.
Throughout this experiment,
we fixed $(\mu_+, \sigma_+, \mu_-, \sigma_-)$ to $(1, 1, -1, 2)$.
Here,
we consider a 1-D thresholded classifier denoted by $h_\theta(x)=1$ if $x \geq \theta$ and $-1$ otherwise.
Given the class prior $\pi_+ \in (0, 1)$,
we generated $m'$ pairwise examples from the above distributions
and apply the proposed class assignment method for a fixed classifier $h_\theta$.
Then, we evaluated whether the estimated class assignment is optimal or not.
Each parameter was set as follows:
$m' \in \{2^1, 2^3, 2^5, 2^7, 2^9\}$,
$\pi_+ = 0.1$,
and $\theta \in \{-3, -2, \ldots, 3 \}$.
For each $(\theta, \pi_+, m')$,
we repeated these data generation processes, class assignment, and evaluation procedure for $\num{10000}$ times.

$\langle\spadesuit\rangle$ The error probabilities are depicted in Figure~\ref{fig:class_assignment}.
We find that the performance of the proposed class assignment method improves as
(i) the number of pairwise examples $m'$ grows
and (ii) the classification error for a given classifier $R_\PN(h_\theta)$ gets away from $\frac{1}{2}$.
These results are aligned with our analysis in Section~\ref{sec:analysis}.
Moreover,
we observed that class assignment improves as the class prior $\pi_+$ becomes farther from $\frac{1}{2}$ in additional experiments in Appendix~\ref{sec:experiments-full}.

\paragraph{Clustering error minimization on a real-world dataset.}

Finally, we show experimental results on a citation network dataset, PubMed-Diabetes.\footnote{
  Available at \url{https://linqs.soe.ucsc.edu/data}.
}
The aim of this experiment is to verify that CIPS is robust enough against real-world noise in pairwise supervision.

We compare CIPS (proposed) with three baselines, MCL (described above), deep metric learning (DML), and SV (supervised).
DML combines metric learning and $k$-means clustering:
we first train embeddings so that their $\ell_2$ distances are close for similar pairs and vice versa,
and apply $k$-means clustering on the embeddings.
More implementation details are deferred to Appendix~\ref{sec:experiments-full}.
The results are reported in Table~\ref{table:pubmed}, from which we can see that
CIPS obtained a meaningful classifier even under the presence of real-world noise,
and worked comparably to MCL and better than DML.

\section{Conclusion}
\label{sec:conclusion}

In this paper,
we presented the underlying relationship between similarity learning and binary classification.
Eventually, the two-step similarity learning procedure for binary classification with only pairwise supervision was obtained.
Our similarity learning can elicit the underlying decision boundary
and is less affected by the class prior.
The post-processing class assignment is less costly than training a new classifier.
Our framework can be related to many existing similarity learning methods with specific losses.
It remains open to discuss the more flexible similarity model and the parallel connection for multi-class classification, in order to fully understand what knowledge we can elicit from similarity information.

\subsubsection*{Acknowledgements}
HB was supported by JSPS KAKENHI Grant Number 19J21094.
MS was supported by JST AIP Acceleration Research Grant Number JPMJCR20U3 and the Institute for AI and Beyond, UTokyo.

\bibliography{ref.bib}
\bibliographystyle{abbrvnat}

\newpage

\appendix
\onecolumn
\section{Proofs of Theorems and Lemmas}
\label{sec:proof}
In this section,
we provide complete proofs for
Theorem~\ref{theo:relationship},
Theorem~\ref{theo:class_assignment},
Lemma~\ref{lemma:surrogate_calibration},
Lemma~\ref{lemma:est_clustering_error_minimization},
and Lemma~\ref{lemma:est_class_assingment}.

\subsection{Proof of Theorem~\ref{theo:relationship}}
We derive an equivalent expression of the pairwise classification error $R_\SD$  as follows.
\begin{equation}
\label{eq:r_sd_pn_relation}
\begin{split}
    R_\SD(h)
    & = \E_{(X, Y) \sim p(\x, y)} \E_{(X', Y') \sim p(\x, y)} \left[\indicator \{ h(X) \cdot h(X') \neq Y Y' \right \}] \\
    & = \E_{(X, Y) \sim p(\x, y)} \E_{(X', Y') \sim p(\x, y)} \left[\indicator \{ h(X) \neq Y \} \indicator\{ h(X') = Y' \right \}] \\
    & \quad + \E_{(X, Y) \sim p(\x, y)} \E_{(X', Y') \sim p(\x, y)} \left[\indicator \{ h(X) = Y \} \indicator\{ h(X) \neq Y' \right \}] \\
    & = \E_{(X, Y) \sim p(\x, y)} \left[\indicator \{ h(X) \neq Y \} \right] \E_{(X', Y') \sim p(\x, y)} \left[ \indicator\{ h(X') = Y' \right \}] \\
    & \quad + \E_{(X, Y) \sim p(\x, y)} \left[\indicator \{ h(X) = Y \} \right] \E_{(X', Y') \sim p(\x, y)} \left[ \indicator\{ h(X') \neq Y' \right \}] \\
    & = 2 \E_{(X, Y) \sim p(\x, y)} \left[\indicator \{ h(X) \neq Y \} \right] \E_{(X', Y') \sim p(\x, y)} \left[ \indicator\{ h(X') = Y' \right \}] \\
    & = 2 \, R_\PN(h) \left(1 - R_\PN(h) \right).
\end{split}
\end{equation}
We can transform the above equation as
\begin{equation}
\label{eq:r_pn_plus}
    R_\PN(h) = \frac{1}{2} \pm \frac{\sqrt{1-2R_\SD(h)}}{2}.
\end{equation}
Then,
we also have
\begin{equation}
\label{eq:r_pn_minus}
    R_\PN(-h) = 1 - R_\PN(h) = \frac{1}{2} \mp \frac{\sqrt{1-2R_\SD(h)}}{2}.
\end{equation}
By combining the results in Eqs.~\eqref{eq:r_pn_plus} and \eqref{eq:r_pn_minus},
we finally obtain Eq.~\eqref{eq:relationship},
which completes the proof of Theorem~\ref{theo:relationship}.
Remark that $0 \le R_\SD(h) \le \frac{1}{2}$ is evident from Eq.~\eqref{eq:r_sd_pn_relation} because of $0 \le R_\PN(h) \le 1$.
\qed

\subsection{Proof of Theorem~\ref{theo:class_assignment}}
The optimal sign $s^*$ can be written as
\begin{equation}
\label{eq:optimal_sign_another}
    s^* = \argmin_{s \in \{\pm 1\}} R_\PN(s \cdot h) = \sign \left(R_\PN(-h) - R_\PN(h) \right).
\end{equation}
According to \citet{shimada2019classification},
$R_\PN$ is equivalently expressed as follows.
\begin{lemma}[Theorem 1 in \citet{shimada2019classification}]
\label{lemma:sd_risk}
Assume that $\pi_+ \neq \frac{1}{2}$.
Then,
the pointwise classification error for a given classifier  $h: \mathcal{X} \rightarrow \mathcal{Y}$ can be equivalently represented as
\begin{align}
    & R_\PN(h) \nonumber \\
    & = \E_{(X, Y) \sim p(\x, y)} \E_{(X', Y') \sim p(\x, y)}
    \left[ \frac{\indicator\left\{h(X) \neq Y Y' \right\} + \indicator\left\{h(X') \neq Y Y' \right\}}{2 \ (2 \pi_+ - 1)} \right] - \frac{1 - \pi_+}{2\pi_+ -1}.
    \label{eq:sd_risk}
\end{align}
\end{lemma}
By plugging Eq.~\eqref{eq:sd_risk} into Eq.~\eqref{eq:optimal_sign_another}, we obtain
\begin{equation}
\begin{split}
    & R_\PN(-h) - R_\PN(h) \\
    & = \E_{(X, Y) \sim p(\x, y)} \E_{(X', Y') \sim p(\x, y)}
    \left[ \frac{\indicator\left\{-h(X) \neq Y Y' \right\} + \indicator\left\{-h(X') \neq Y Y' \right\}}{2 \ (2 \pi_+ - 1)} \right] \\
    & \quad - \E_{(X, Y) \sim p(\x, y)} \E_{(X', Y') \sim p(\x, y)}
    \left[ \frac{\indicator\left\{h(X) \neq Y Y' \right\} + \indicator\left\{h(X') \neq Y Y' \right\}}{2 \ (2 \pi_+ - 1)} \right] \\
    & = \E_{(X, Y) \sim p(\x, y)} \E_{(X', Y') \sim p(\x, y)}
    \left[ \frac{1 - 2 \cdot \indicator\left\{h(X) \neq Y Y' \right\}
    +1 - 2 \cdot \indicator\left\{h(X') \neq Y Y' \right\}}{2 \ (2 \pi_+ - 1)} \right] \\
    & = \frac{1}{2 \pi_+ - 1} \E_{(X, Y) \sim p(\x, y)} \E_{(X', Y') \sim p(\x, y)}
    \left[ 1 - \indicator\left\{h(X) \neq Y Y' \right\} - \indicator\left\{h(X') \neq Y Y' \right\} \right]  \\
    & = \frac{1}{2 \pi_+ - 1} (1 - 2Q(h)).
\end{split}
\end{equation}
Thus, we derive the following result.
\begin{equation}
\begin{split}
    s^*_h = \sign \left(R_\PN(-h) - R_\PN(h) \right)
    & = \sign \left( \frac{1}{2 \pi_+ - 1} \right) \cdot \sign(1 - 2Q(h)) \\
    & = \sign (2 \pi_+ - 1 ) \cdot \sign(1 - 2Q(h)),
\end{split}
\end{equation}
which completes the proof of Theorem~\ref{theo:class_assignment}.
Note that $s^*$ can be either $\pm 1$ when $Q(h) = \frac{1}{2}$,
which is equivalent to $R_\PN(h) = R_\PN(-h) = \frac{1}{2}$.
Here we arbitrarily set to $s^* = -\sign(2\pi_+ - 1)$ in this case.
\qed

\subsection{Proof of Lemma~\ref{lemma:surrogate_calibration}}

We introduce the following notation:
\begin{align*}
    S_\PN^\ell(\alpha, \eta) &\defeq \eta\ell(\alpha,+1) + (1-\eta)\ell(\alpha,-1), \\
    H_\PN^\ell(\eta) &\defeq \inf_{\alpha \in \R} S_\PN^\ell(\alpha, \eta), \\
    H_\PN^{\ell,-}(\eta) &\defeq \inf_{\alpha:\alpha(2\eta-1) \leq 0} S_\PN^\ell(\alpha, \eta).
\end{align*}
$S_\PN^\ell$ represents the conditional $\ell$-risk in the following sense:
\begin{equation*}
    \E_X[S_\PN^\ell(f(X), p(Y=+1|X))] = R_\PN^\ell(f),
\end{equation*}
where
\begin{equation*}
    R_\PN^\ell(f) \defeq \E_{(X,Y) \sim p(\x,y)}[\ell(f(X), Y)].
\end{equation*}
Define the function $\psi_\PN:[0,1] \to [0,+\infty)$ by $\psi_\PN = \tilde\psi_\PN^{\star\star}$,
where $\tilde\psi_\PN^{\star\star}$ is the Fenchel-Legendre biconjugate of $\tilde\psi_\PN$,
and
\begin{equation*}
    \tilde\psi_\PN(\varepsilon) \defeq H_\PN^{\ell,-}\left(\frac{1+\varepsilon}{2}\right) - H_\PN^\ell\left(\frac{1+\varepsilon}{2}\right).
\end{equation*}
$\psi_\PN$ corresponds to $\psi$-transform introduced by \citet{bartlett2006convexity} exactly.

We will show that the statement of the lemma is satisfied by $\psi = \psi_\PN$
based on the \emph{calibration analysis}~\citep{steinwart2007compare}.
We further introduce the following notation:
\begin{align*}
    S_\SD(\alpha, \alpha', \eta, \eta')
    &\defeq \eta\eta'\indicator\{\sign(\alpha)\sign(\alpha') \ne +1\} \\
    &\quad +\eta(1-\eta')\indicator\{\sign(\alpha)\sign(\alpha') \ne -1\} \\
    &\quad +(1-\eta)\eta'\indicator\{\sign(\alpha)\sign(\alpha') \ne -1\} \\
    &\quad +(1-\eta)(1-\eta')\{\sign(\alpha)\sign(\alpha') \ne +1\},
    \\
    S_\SD^\ell(\alpha, \alpha', \eta, \eta')
    &\defeq \eta\eta'\ell(\alpha\alpha', +1) + \eta(1-\eta')\ell(\alpha\alpha, -1) \\
    &\quad + (1-\eta)\eta'\ell(\alpha\alpha, -1) + (1-\eta)(1-\eta')\ell(\alpha\alpha, +1),
    \\
    H_\SD(\eta, \eta') &\defeq \inf_{\alpha, \alpha' \in \R} S_\SD(\alpha, \alpha', \eta, \eta'), \\
    H_\SD^\ell(\eta, \eta') &\defeq \inf_{\alpha, \alpha' \in \R} S_\SD^\ell(\alpha, \alpha', \eta, \eta').
\end{align*}
$S_\SD^\ell$ represents the conditional $\ell$-risk in the following sense:
\begin{equation*}
    \E_{X,X'}[S_\SD^\ell(f(X), f(X'), p(Y=+1|X), p(Y'=+1|X'))] = R_\SD^\ell(f),
\end{equation*}
and
\begin{equation*}
    \E_{X,X'}[S_\SD(f(X), f(X'), p(Y=+1|X), p(Y'=+1|X'))] = R_\SD(\sign \circ f).
\end{equation*}
Let $\tilde\psi_\SD: [0,1] \to [0,+\infty)$ be the \emph{calibration function}~\citep[Lemma~2.16]{steinwart2007compare} defined by
\begin{align*}
    \tilde\psi_\SD(\varepsilon)
    &\defeq \inf_{\eta,\eta \in [0,1]} \inf_{\alpha, \alpha' \in \R} S_\SD^\ell(\alpha, \alpha', \eta, \eta') - H_\SD^\ell(\eta, \eta') \\
    & \hspace{100pt} \text{ s.t. }
    S_\SD(\alpha, \alpha', \eta, \eta') - H_\SD(\eta, \eta') \geq \varepsilon.
\end{align*}
By the consequence of Lemma~2.9 of \citet{steinwart2007compare},
$\tilde\psi_\SD(\varepsilon) > 0$ for all $\varepsilon > 0$ implies that $R_\SD^\ell(f) \to R_\SD^* \implies R_\SD(\sign \circ f) \to R_\SD^*$.
Further, under this condition, Theorem~2.13 of \citet{steinwart2007compare} implies that $\tilde\psi_\SD$ is non-decreasing, invertible, and satisfies
\begin{equation*}
    \tilde\psi_\SD^{\star\star}(R_\SD(\sign \circ f) - R_\SD^*) \leq R_\SD^\ell(f) - R_\SD^{\ell,*}
\end{equation*}
for any measurable function $f$.
Hence, it is sufficient to show that $\tilde\psi_\SD(\varepsilon) > 0$ for all $\varepsilon > 0$.
Indeed, $\tilde\psi_\SD = \tilde\psi_\PN$, and $\tilde\psi_\PN(\varepsilon) > 0$ for all $\varepsilon > 0$
because $\ell$ is classification-calibrated~\citep[Lemma~2]{bartlett2006convexity}.
From now on, we will see $\tilde\psi_\SD = \tilde\psi_\PN$.

First, we simplify the constraint part of $\tilde\psi_\SD$.
Since
\begin{align*}
    S_\SD(\alpha, \alpha', \eta, \eta')
    &= (1 - \eta - \eta' + 2\eta\eta')\indicator\{\sign(\alpha)\sign(\alpha') = -1\} \\
    &\quad+ (\eta + \eta' - 2\eta\eta')\indicator\{\sign(\alpha)\sign(\alpha') = +1\}
    \\
    &= \tilde\eta\indicator\{\sign(\alpha)\sign(\alpha') = +1\} + (1 - \tilde\eta)\indicator\{\sign(\alpha)\sign(\alpha') = -1\},
\end{align*}
where $\tilde\eta \defeq 1 - \eta - \eta' + 2\eta\eta'$,
we have $H_\SD(\eta, \eta') = \min\{\tilde\eta, 1 - \tilde\eta\}$.
Similarly,
\begin{align*}
    S_\SD^\ell(\alpha, \alpha', \eta, \eta')
    = \tilde\eta\ell(\alpha\alpha', +1) + (1 - \tilde\eta)\ell(\alpha\alpha', -1).
\end{align*}
With slight abuse of notation, we may write $S_\SD(\alpha, \alpha', \tilde\eta) = S_\SD(\alpha, \alpha', \eta, \eta')$ (same for $S_\SD^\ell$, $H_\SD$, and $H_\SD^\ell$).
By simple algebra, we obtain
\begin{align*}
    S_\SD(\alpha, \alpha', \tilde\eta) - H_\SD(\tilde\eta)
    &= |2\tilde\eta - 1| \cdot \indicator\{(2\tilde\eta - 1)\sign(\alpha)\sign(\alpha') \leq 0\}.
\end{align*}
Noting that $\tilde\eta$ ranges over $[0, 1]$ with $\eta, \eta' \in [0, 1]$,
we have
\begin{equation*}
\begin{split}
    \tilde\psi_\SD(\varepsilon)
    &= \inf_{\tilde\eta \in [0, 1]} \inf_{\alpha, \alpha' \in \R} S_\SD^\ell(\alpha, \alpha', \tilde\eta) - H_\SD^\ell(\tilde\eta) \\
    &\qquad \text{s.t.} \quad
    |2\tilde\eta - 1| \cdot \indicator\{(2\tilde\eta - 1)\sign(\alpha)\sign(\alpha') \leq 0\} \geq \varepsilon.
\end{split}
\end{equation*}
If $\varepsilon = 0$, $\tilde\psi_\SD(0) = 0$ and the infimum is attained by $\tilde\eta = \frac{1}{2}$ and arbitrary $\alpha$ and $\alpha'$.
If $\varepsilon > 0$, $\tilde\eta = \frac{1}{2}$ cannot satisfy the constraint.
Hence, we assume $\tilde\eta \ne \frac{1}{2}$ from now on.
When $\tilde\eta > \frac{1}{2}$, the constraint reduces to
\begin{align*}
    \left\{ \alpha\alpha' \leq 0 \wedge (\alpha, \alpha') \ne (0, 0) \right\} \vee \tilde\eta \geq \frac{1+\varepsilon}{2}.
\end{align*}
Since $S_\SD^\ell$ contains $\alpha$ and $\alpha'$ only in the form of $\alpha\alpha'$,
the infimum over $\{\alpha, \alpha' \in \R \mid \alpha\alpha' \leq 0 \wedge (\alpha, \alpha') \ne (0, 0)\}$ is equal to that over $\{\alpha, \alpha' \in \R \mid \alpha\alpha' \leq 0\}$.
If we write $\alpha\alpha' \defeq \tilde\alpha$, then
\begin{align*}
    \tilde\psi_\SD(\varepsilon)
    &= \inf_{\tilde\eta \in \left[\frac{1+\varepsilon}{2}, 1\right]} \inf_{\tilde\alpha \in \R: \tilde\alpha \leq 0} S_\SD^\ell(\alpha, \alpha', \tilde\eta) - H_\SD^\ell(\tilde\eta) \\
    &= \inf_{\tilde\eta \in \left[\frac{1+\varepsilon}{2}, 1\right]} \inf_{\tilde\alpha \in \R: \tilde\alpha \leq 0} S_\PN^\ell(\tilde\alpha, \tilde\eta) - H_\PN^\ell(\tilde\eta) \\
    &= \inf_{\tilde\alpha \in \R: \tilde\alpha \leq 0} S_\PN^\ell\left(\tilde\alpha, \frac{1+\varepsilon}{2}\right) - H_\PN^\ell\left(\frac{1+\varepsilon}{2}\right) \\
    &= H_\PN^{\ell,-}\left(\frac{1+\varepsilon}{2}\right) - H_\PN^\ell\left(\frac{1+\varepsilon}{2}\right) \\
    &= \tilde\psi_\PN(\varepsilon).
\end{align*}
When $\tilde\eta < \frac{1}{2}$, $\tilde\psi_\SD = \tilde\psi_\PN$ can be shown in the same way.
Hence, the statement is proven.
\qed

\subsection{Proof of Lemma~\ref{lemma:est_clustering_error_minimization}}
We start by introducing the following statement.
\begin{lemma}
\label{lemma:alpha_beta}
For real values $\alpha$ and $\beta$ satisfying $0 \leq \alpha \leq \beta \leq 1$,
we have
\begin{equation}
    \sqrt{\beta} - \sqrt{\alpha} \leq \sqrt{\beta - \alpha}.
\end{equation}
\end{lemma}
\begin{proof}
\begin{equation}
    \begin{split}
        (\beta - \alpha) - (\sqrt{\beta} - \sqrt{\alpha})^2
        = 2 \sqrt{\alpha \beta}  -2 \alpha
        = 2\sqrt{\alpha} \left(\sqrt{\beta} - \sqrt{\alpha} \right)
        \geq 0,
    \end{split}
\end{equation}
Thus we have $(\beta - \alpha) \geq (\sqrt{\beta} - \sqrt{\alpha})^2$, which completes the proof of Lemma~\ref{lemma:alpha_beta}.
\end{proof}

With this lemma,
an excess risk on clustering error can be connected with that on pairwise classification error as follows.
From the equation in Eq.~~\eqref{eq:relationship},
we have
\begin{equation}
    R_\CE^* = \frac{1}{2} - \frac{\sqrt{1-2  R_\SD^*}}{2}.
\end{equation}
Thus,
we can bound excess risk on the clustering error as follows.
\begin{equation}
\begin{split}
\label{eq:ce_sd_bound}
    R_{\CE}(\sign \circ \hat{f}) - R_{\CE}^*
    & = \left(\frac{1}{2} - \frac{\sqrt{1-2R_\SD(\sign \circ \hat{f})}}{2}\right) - \left(\frac{1}{2} - \frac{\sqrt{1-2R_\SD^*}}{2}\right) \\
    & = \frac{1}{2} \left\{ \sqrt{1-2R_\SD^*} - \sqrt{1-2R_\SD(\sign \circ \hat{f})} \right\}\\
    & \leq \sqrt{\frac{{{R_\SD(\sign \circ \hat{f}) - R_\SD^*}}}{2}} \\
    & \leq \sqrt{\frac{1}{2} \psi^{-1} \left(R_\SD^\ell (\hat{f}) - R_\SD^{\ell,*} \right) },
\end{split}
\end{equation}
where Lemma~\ref{lemma:alpha_beta} and Lemma~\ref{lemma:surrogate_calibration} were applied to obtain the penultimate and the last inequalities, respectively.
The excess risk with respect to pairwise surrogate risk,
i.e., $R_\SD^\ell (\hat{f}) - R_\SD^{\ell,*}$,
can be decomposed into \emph{approximation error} and \emph{estimation error} as
\begin{equation}
\label{eq:bias_variance_decomposition}
    R_\SD^\ell (\hat{f}) - R_\SD^{\ell*}
    = \underbrace{R_\SD^\ell (f^*) - R_\SD^{\ell*}}_{\text{approximation error}} + \underbrace{R_\SD^\ell (\hat{f}) - R_\SD^\ell (f^*)}_{\text{estimation error}},
\end{equation}
where $f^*$ is the minimizer of $R_\SD^\ell(f)$ in a specified function space $\mathcal{F}$.
Now,
we provide the following upper bound for the estimation error with the Rademacher complexity.
\begin{lemma}
\label{lemma:est_surrogate_sd_risk}
Let $f^* \in \mathcal{F}$ be a minimizer of $R_\SD^\ell$,
and $\hat{f} \in \mathcal{F}$ be a minimizer of
the empirical risk $\widehat{R}_\SD^\ell$.
Assume that the loss function $\ell$ is $\rho$-Lipschitz function with respect to the first argument ($0 < \rho < \infty$),
and all functions in the model class $\mathcal{F}$ are bounded,
i.e., there exists an constant $C_b$ such that $\| f \|_\infty \leq C_b$ for any $f \in \mathcal{F}$.
Let $C_\ell \defeq \sup_{t \in \{\pm1 \}} \ell(C_b^2, t)$.
For any $\delta > 0$, with probability at least $1-\delta$,
\begin{equation}
\label{eq:estimation_error_bound}
\begin{split}
     R_\SD^\ell(\hat{f}) - R_\SD^\ell(f^*)
    & \leq 4\rho \mathfrak{R}_{m_1}(\mathcal{F}) + \sqrt{\frac{2 C_\ell^2 \log \frac{2}{\delta}}{m_1}}.
\end{split}
\end{equation}
\end{lemma}
\begin{proof}
The estimation error can be bounded as
\begin{equation}
\label{eq:estimation_to_generalization}
\begin{split}
     R_\SD^\ell(\hat{f}) - R_\SD^\ell(f^*)
    &  \leq \left( R_\SD^\ell(\hat{f}) - \hat{R}_\SD^\ell(\hat{f}) \right)+ \left( \hat{R}_\SD^\ell(f^*) - R_\SD^\ell(f^*) \right) \\
    & \leq 2 \sup_{f \in \mathcal{F}} \left| R_\SD^\ell(\hat{f}) - \hat{R}_\SD^\ell(\hat{f}) \right|.
\end{split}
\end{equation}
With the Rademacher complexity,
the following inequalities hold with probability at least $1-\delta$.
\begin{equation}
\label{eq:uniform_deviation}
    \left| R_\SD^\ell(\hat{f}) - \hat{R}_\SD^\ell(\hat{f}) \right| \leq 2 \mathfrak{R}_{m_1}(\ell \circ \mathcal{F}) + \sqrt{\frac{C_\ell^2 \log \frac{2}{\delta}}{2m_1}},
\end{equation}
where $\ell \circ \mathcal{F}$ indicates a class of composite functions defined by $\{\ell \circ f \mid f \in \mathcal{F} \}$.
By applying Talagrand's lemma,
the Rademacher complexity of $\ell \circ \mathcal{F}$ can be bounded as
\begin{equation}
\label{eq:talagrand_lemma}
    \mathfrak{R}_{m_1}(\ell \circ \mathcal{F}) \leq \rho \mathfrak{R}_{m_1}(\mathcal{F}).
\end{equation}
The proofs of Eqs.~\eqref{eq:uniform_deviation} and \eqref{eq:talagrand_lemma} can be found in \citet[Theorem~3.1 and Lemma~4.2]{mohri2018foundations}, respectively.
By plugging Eqs.~\eqref{eq:uniform_deviation} and \eqref{eq:talagrand_lemma} into Eq.~\eqref{eq:estimation_to_generalization},
we obtain the result in Eq.~\eqref{eq:estimation_error_bound}.
\end{proof}
By combining Eqs.~\eqref{eq:ce_sd_bound}, \eqref{eq:bias_variance_decomposition} and Lemma~\ref{lemma:est_surrogate_sd_risk},
we obtain the following inequality with probability at least $1-\delta$,
\begin{equation}
    R_\CE (\sign \circ \hat{f}) - R_\CE^* \leq \sqrt{\frac{1}{2} \psi^{-1}\left (
    R_\SD^\ell(f^*) - R_\SD^{\ell,*} +
    4\rho \mathfrak{R}_{m_1}(\mathcal{F}) + \sqrt{\frac{2 C_\ell^2 \log \frac{2}{\delta}}{m_1}} \right)}.
\end{equation}
\qed

\subsection{Proof of Lemma~\ref{lemma:est_class_assingment}}
We first derive a sufficient condition for the proposed class assignment fails.
Let $\hat{s}$ be a estimated class assignment for a given hypothesis $h: \mathcal{X} \rightarrow \mathcal{Y}$.
\begin{equation}
\begin{split}
    \mathrm{Pr}\left( \hat{s} \neq \argmin_{s \in \{\pm 1\}} R_\PN(s \cdot h)  \right)
    & = \mathrm{Pr}\left(\sign \left(1- 2\hat{Q}(h) \right)  \neq \sign \left(1 - 2 Q(h) \right)  \right) \\
    & =
    \begin{cases}
    \mathrm{Pr} \left( 2 \hat{Q}(h) - 1 > 0 \right) & (1 - 2 Q(h) > 0 ), \\
    \mathrm{Pr} \left( 2 \hat{Q}(h) - 1 \leq 0 \right) & (\text{otherwise})
    \end{cases} \\
    & =
    \begin{cases}
    \mathrm{Pr} \left(\hat{Q}(h) - Q(h) > \frac{1}{2} - Q(h) \right) & (1 - 2 Q(h) > 0 ), \\
    \mathrm{Pr} \left( Q(h) - \hat{Q}(h) \geq Q(h) - \frac{1}{2} \right) & (\text{otherwise})
    \end{cases}
\end{split}
\end{equation}
By applying Hoeffding's inequality~\cite{hoeffding1963probability},
we obtain the following bounds.
\begin{gather}
    \mathrm{Pr} \left(\hat{Q}(h) - Q(h) > \frac{1}{2} - Q(h) \right) \leq \exp \left( -2 m_2 \left(Q(h) - \frac{1}{2} \right)^2 \right), \\
    \mathrm{Pr} \left({Q}(h) - \hat{Q}(h) \geq Q(h) - \frac{1}{2} \right) \leq \exp \left( -2 m_2 \left(Q(h) - \frac{1}{2} \right)^2 \right),
\end{gather}
where $m_2$ is the number of pairwise examples to compute $\hat{Q}(h)$.
Therefore,
we can bound the error probability of the proposed class assignment method regardless of the value of $Q(h)$ as
\begin{equation}
\label{eq:class_assingment_error_bound}
    \mathrm{Pr}\left( \hat{s} \neq \argmin_{s \in \{\pm 1\}} R_\PN(s \cdot h)  \right) \leq \exp \left( -2 m_2 \left(Q(h) - \frac{1}{2} \right)^2 \right).
\end{equation}
Now,
we further explore how the term $Q(h) - \frac{1}{2}$ can be expressed.
From the definition of $Q$ and the equivalent risk expression in Eq.~\eqref{eq:sd_risk},
we have
\begin{equation}
    Q(h) = (2 \pi_+ - 1) R_\PN(h) + 1- \pi_+.
\end{equation}
Therefore,
\begin{equation}
\label{eq:Q_minus_half}
    Q(h) - \frac{1}{2} = (2 \pi_+ - 1) \left(R_\PN(h) - \frac{1}{2} \right).
\end{equation}
By plugging Eq.~\eqref{eq:Q_minus_half} into Eq.~\eqref{eq:class_assingment_error_bound},
we finally obtain
\begin{equation}
    \mathrm{Pr}\left( \hat{s} \neq \argmin_{s \in \{\pm 1\}} R_\PN(s \cdot h)  \right) \leq \exp \left( -\frac{m_2}{2} (2 \pi_+ - 1)^2 \left(2 R_\PN(h) - 1 \right)^2 \right),
\end{equation}
which completes the proof of Lemma~\ref{lemma:est_class_assingment}.
\qed

\section{Discussion on Class Assignment}
\label{subsec:class-assignment}

In this section, we discuss the impossibility of recovering the class assignment
only with unlabeled validation data.
Given a real-valued prediction function $f: \mathcal{X} \to \R$ and the class prior $\pi_+$,
we consider the following class assignment strategy instead of the proposed method:
\begin{equation}
    \tilde{s} \defeq \sign \left(2 \pi_+ - 1 \right) \cdot \sign \left( \E_{X \sim p(\x)} \left[ \sign\left( f(X) \right) \right]\right).
\end{equation}
Our aim is to estimate the optimal class assignment $s^* = \argmin_{s \in \{\pm 1\}} R_\PN \left(s \cdot \sign \circ f\right)$,
which can be expressed by
\begin{equation}
\begin{split}
    s^*
    & = \sign \left( R_\PN \left(- \sign \circ  f \right) - R_\PN \left(\sign \circ  f \right)  \right) \\
    & = \sign \left( (1- R_\PN \left(\sign \circ  f \right)) - R_\PN \left(\sign \circ  f \right)  \right) \\
    & = \sign \left(1 - 2 R_\PN (\sign \circ f)  \right).
\end{split}
\end{equation}
Thus, the following condition is necessary and sufficient for $\tilde s = s^*$:
\begin{equation}
\label{eq:supp:class_assignment_iff_condition}
\begin{split}
    \underbrace{\sign \left(2 \pi_+ - 1 \right) \cdot \sign \left( \E_{X \sim p(\x)} \left[ \sign\left( f(X) \right) \right]\right)}_{= \tilde s} \cdot \underbrace{\sign \left(1 - 2 R_\PN (\sign \circ f) \right)}_{= s^*} > 0.
\end{split}
\end{equation}
We will investigate whether this condition always holds or not.
Denote
\begin{gather}
    R^+_\PN \defeq \pi_+ \E_{X \sim p(\x \mid y=+1)} \left[\indicator \left\{ \sign\left(f(X)\right) \neq +1 \right\} \right], \\
    R^-_\PN \defeq (1-\pi_+) \E_{X \sim p(\x \mid y=-1)} \left[\indicator \left\{ \sign\left(f(X)\right) \neq -1 \right\} \right].
\end{gather}
Note that $0 \leq R^+_\PN \leq \pi_+$ and $0 \leq R^-_\PN \leq 1-\pi_+$ always hold.
Now, we have
\begin{equation}
    \begin{split}
        \E_{X \sim p(\x)} \left[\indicator\left\{ \sign\left(f(X)\right) = +1 \right\}\right]
        &= \pi_+ \E_{X \sim p(\x \mid y=+1)}\left[\indicator\left\{ \sign\left(f(X)\right) = +1 \right\}\right] \\
          & \qquad + (1-\pi_+) \E_{X \sim p(\x \mid y=-1)}\left[\indicator\left\{ \sign\left(f(X)\right) = +1 \right\}\right] \\
        &= \pi_+ \E_{X \sim p(\x \mid y=+1)}\left[\left(1 - \indicator\left\{ \sign\left(f(X)\right) = -1 \right\}\right)\right] \\
          & \qquad + (1-\pi_+) \E_{X \sim p(\x \mid y=-1)}\left[\indicator\left\{ \sign\left(f(X)\right) = +1 \right\}\right] \\
        &= -\pi_+ \E_{X \sim p(\x \mid y=+1)}\left[\indicator\left\{ \sign\left(f(X)\right) = -1 \right\}\right] \\
          & \qquad + (1-\pi_+) \E_{X \sim p(\x \mid y=-1)}\left[\indicator\left\{ \sign\left(f(X)\right) = +1 \right\}\right] + \pi_+ \\
        &= -R_\PN^+ + R_\PN^- + \pi_+.
    \end{split}
\end{equation}
Similarly, we have
\begin{equation}
    \E_{X \sim p(\x)} \left[\indicator\left\{ \sign\left(f(X)\right) = -1 \right\}\right]
    = R_\PN^+ - R_\PN^- + 1 - \pi_+.
\end{equation}
By combining them, the following expression can be obtained.
\begin{equation}
    \begin{split}
    \E_{X \sim p(\x)} \left[ \sign\left( f(X) \right) \right]
    & = \E_{X \sim p(\x)} \left[ \indicator \left\{ \sign\left(f(X)\right) = +1 \right\} \right] - \E_{X \sim p(\x)} \left[ \indicator \left\{ \sign\left(f(X)\right) = -1 \right\} \right] \\
    & = -2 R^+_\PN + 2 R^-_\PN+ 2 \pi_+ - 1.
    \end{split}
\end{equation}
Hence, the necessary and sufficient condition \eqref{eq:supp:class_assignment_iff_condition} is rewritten as
\begin{equation}
    \label{eq:supp:class_assignment_iff_condition:1}
    \sign \left(2 \pi_+ - 1 \right) \cdot \sign \left( - 2R^+_\PN + 2R^-_\PN+ 2\pi_+ - 1 \right)
    \cdot \sign \left(1 - 2 R^+_\PN - 2 R^-_\PN \right) > 0.
\end{equation}
This condition is satisfied when $\pi_+$, $R_\PN^+$, and $R_\PN^-$ satisfy any of the following conditions.
\begin{itemize}
    \item $\pi_+ \geq \frac{1}{2}$, $R^-_\PN\geq R^+_\PN + \frac{1}{2} - \pi_+$, and $R^-_\PN\leq -R^+_\PN + \frac{1}{2}$,
    \item $\pi_+ \geq \frac{1}{2}$, $R^-_\PN< R^+_\PN + \frac{1}{2} - \pi_+$, and $R^-_\PN> -R^+_\PN + \frac{1}{2}$,
    \item $\pi_+ < \frac{1}{2}$, $R^-_\PN\geq R^+_\PN + \frac{1}{2} - \pi_+$, and $R^-_\PN> -R^+_\PN + \frac{1}{2}$,
    \item $\pi_+ < \frac{1}{2}$, $R^-_\PN< R^+_\PN + \frac{1}{2} - \pi_+$, and $R^-_\PN\leq -R^+_\PN + \frac{1}{2}$.
\end{itemize}

\begin{figure}[ht]
  \centering
  \subfigure[$\pi_+ > \frac{1}{2}$]{\includegraphics[width=0.3\linewidth]{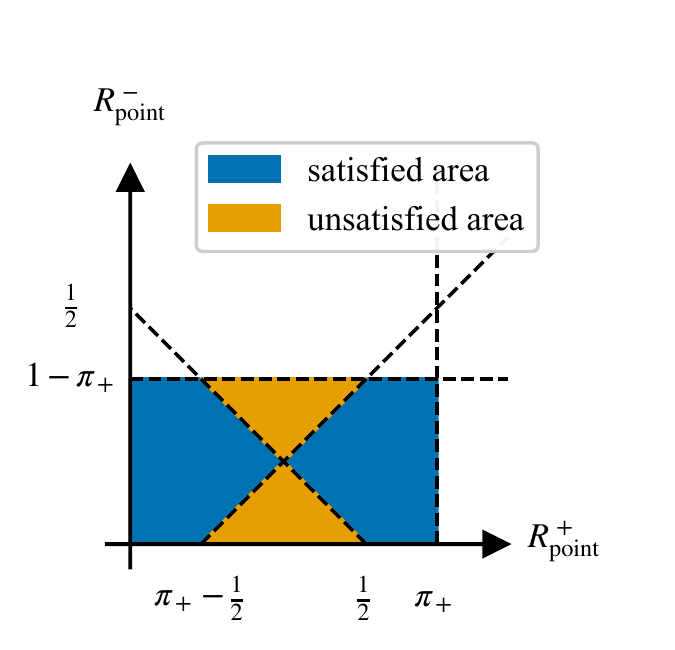}}
  \subfigure[$\pi_+ < \frac{1}{2}$]{\includegraphics[width=0.3\linewidth]{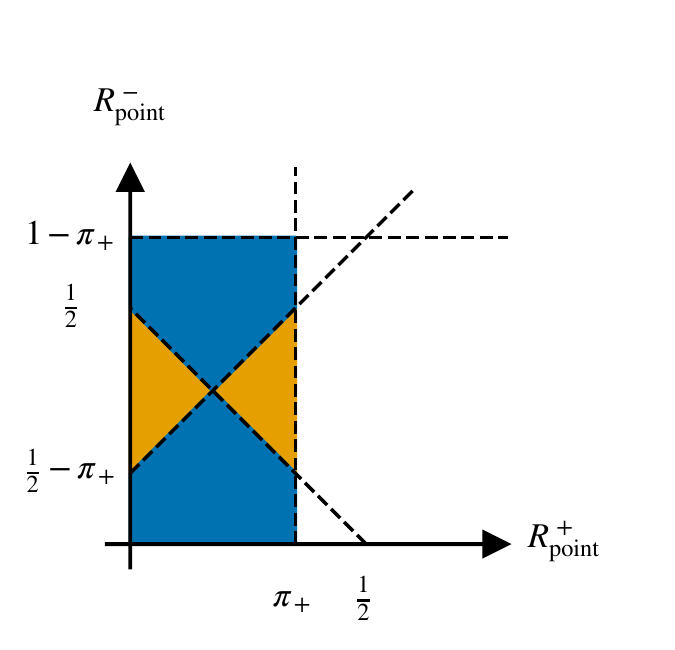}}
  \caption{
    Illustration of the areas corresponding to the condition \eqref{eq:supp:class_assignment_iff_condition:1} (highlighted with blue).
    We have $\tilde s = s^*$ in the blue areas and otherwise in the orange areas.
  }
  \label{fig:supp:success_of_class_assignment}
\end{figure}

The conditions \eqref{eq:supp:class_assignment_iff_condition:1} are depicted in Figure~\ref{fig:supp:success_of_class_assignment}.
As can be seen from this figure, for any binary classification problem (i.e., for any $\pi_+$),
there exists a case where the class assignment with unlabeled data fails ($\tilde s \ne s^*$).

\section{Extension to semi-supervised learning}
\label{subsec:additional-supervision}
In real-world applications,
we may face the situation where a large amount of unlabeled data are available along with pairwise data.
Similarly to existing weakly-supervised classification frameworks such as positive-negative-unlabeled classification~\citep{sakai2017semi} and similar-dissimilar-unlabeled classification~\citep{shimada2019classification},
we can easily incorporate unlabeled data for the estimation of $R_\SD^\ell$.
\begin{theorem}
For non-negative real values $(\gamma_1, \gamma_2, \gamma_3)$ that satisfies $\gamma_1 + \gamma_2 + \gamma_3 = 1$,
the risk $R_\SD^\ell(f)$ can be equivalently expressed as:
\begin{equation}
\label{eq:risk_semi}
\begin{split}
    &(\pi_+^2 + \pi_-^2) \E_{(X, X') \sim p(\x, \x' \mid y y' = +1)} \left[ (\gamma_1 + \gamma_2) \ell(f(X)f(X'), +1) - \gamma_2 \ell(f(X)f(X'), -1) \right] \\
    & + 2\pi_+ \pi_- \E_{(X, X') \sim p(\x, \x' \mid y y' = -1)} \left[ (\gamma_1 + \gamma_3) \ell(f(X)f(X'), -1) - \gamma_3 \ell(f(X)f(X'), +1) \right] \\
    & + \E_{(X, X') \sim p(\x, \x')} \left[ \gamma_3 \ell(f(X)f(X'), +1) + \gamma_2 \ell(f(X)f(X'), -1) \right],
\end{split}
\end{equation}
where $\pi_+$ and $\pi_-$ denote positive and negative class proportions, respectively.
\end{theorem}

With the expression in Eq.~\eqref{eq:risk_semi},
we can use both pairwise supervision and unlabeled data for the empirical estimation of $R_\SD^\ell$.
As well as the similar-unlabeled classification method~\citep{bao2018classification},
our method can be applied with only similar-unlabeled (or dissimilar-unlabeled) data by controlling parameters $(\gamma_1, \gamma_2, \gamma_3)$.

\section{Training with linear model and unhinged Loss}
In general, the optimization problem in Eq.~\eqref{eq:f_hat} is non-convex.
Thus, it is not guaranteed whether we can achieve global optima with gradient descent.
However,
with specific model class and loss function,
we can obtain an optimal solution more efficiently.
Consider the linear model $f(\x)=\w^\top \x$, where $\w \in \R^d$ are parameters.
As a loss function,
we consider the unhinged loss $\ell_{\rm UH}(z, t) \defeq 1 - tz$.
This loss function is originally proposed in \citet{van2015learning} to cope with label noises.
Here we reformulate the optimization problem with linear model and the unhinged loss as follows.
\begin{equation}
\label{eq:opt_unhinge}
    \hat{\w} = \argmin_{\w} \hat{R}_\SD^{\ell_{\rm UH}}(\w), \text{~~~s.t.~~~} \| \w \| = 1,
\end{equation}
where
\begin{equation}
\begin{split}
\hat{R}_\SD^{\ell_{\rm UH}} (\w)
    & \defeq \frac{1}{m} \sum_{(X, X', T) \in \mathfrak{D}_1} \left(1 - T \, \w^\top X  \cdot \w^\top X' \right) \\
    & = 1 - \w^\top \left( \frac{1}{m} \sum_{(X, X', T) \in \mathfrak{D}_1} T X {X'}^\top \right) \w \\
    & = 1 - \w^\top \left( \frac{1}{2m} \sum_{(X, X', T) \in \mathfrak{D}_1} T \left(X {X'}^\top + X' {X}^\top \right) \right) \w \\
    & = 1 - \w^\top M \w,
\end{split}
\end{equation}
where $M$ denotes the Hermitian matrix $\frac{1}{2m} \sum_{(X, X', T) \in \mathfrak{D}_1} T \left(X {X'}^\top + X' {X}^\top \right)$.
The constraint $\| \w \| = 1$ is necessary to prevent the objective function from divergence.
Let $\lambda_1, \ldots, \lambda_d$ be eigenvalues of the matrix $M$ that satisfies $\lambda_1 \geq \dots \geq \lambda_d$,
and $\boldsymbol{v}_1,\ldots,\boldsymbol{v}_d$ be corresponding eigenvectors
that satisfy $\|\boldsymbol{v}_i \| = 1$ for all $i \in \{1, \ldots, d \}$.
The following statement is known as a property of \emph{Rayleigh quotient}~\citep{horn2012matrix}.
\begin{equation}
 \boldsymbol{v}_1 = \argmax_{\w \in \R^d, \| \w \| = 1} \w^\top M \w .
\end{equation}
Thus,
the analytical solution of the constrained optimization problem in Eq.~\eqref{eq:opt_unhinge} is obtained as
\begin{equation}
    \hat{\w} = \argmin_{\w \in \R^d, \|\w \| = 1} \hat{R}_\SD^{\ell_{\rm UH}}(\w)
    = \argmax_{\w \in \R^d, \|\w \| = 1} \w^\top M \w = \boldsymbol{v}_1.
\end{equation}

\section{Full version of experimental results}
\label{sec:experiments-full}
In this section,
we show the implementation details and the full versions of experimental results in Section~\ref{sec:experiments},
which were omitted in the main body due to the limited space.

\paragraph{implementation details (clustering error minimization on benchmark datasets).}
The implementation details of our method (CIPS) and each baseline were as follows.
\begin{itemize}
    \setlength{\leftskip}{-10pt}
    \item CIPS (Ours):
    The empirical pairwise classification risk $R_\SD$ \eqref{eq:f_hat} was computed with the logistic loss.
    The linear model $f(\x) = \w^\top\x + b$ was used.
    The risk was optimized with the stochastic gradient descent (minibatch size: $\num{64}$ / learning rate: $10^{-2}$ / $\ell_2$-regularization parameter: $10^{-4}$ / training epochs: $\num{500}$).

    \item MCL~\citep{hsu2018multi}:
    The loss function is based on the maximum likelihood, that is, the logistic loss as in the original paper.
    The model and optimization setup were the same as CIPS.

    \item SD~\citep{shimada2019classification}:
    Their proposed classification risk was computed with the logistic loss.
    The model and optimization setup were the same as CIPS.

    \item OVPC~\citep{zhang2007value}:
    We followed the authors to use the squared loss and the closed-form minimizer was evaluated.

    \item SSP~\citep{von2007tutorial}:
    Pairwise data were used as hard constraints.
    In order to construct the neighborhood sets for the Laplacian matrix, $\num{5}$-nearest neighbors were used.
    The features are obtained by constraints propagation.
    In order to perform the final $k$-means clustering on the obtained features,
    scikit-learn implementation~\citep{scikit-learn} was used with the default parameters.

    \item CKM~\citep{wagstaff2001constrained}:
    Pairwise data were used as hard constraints.
    Clustering was carried out with $\num{10}$ different random initializations and the best one was reported.
    For each initialization, the number of maximum iterations was set to $\num{300}$ and the tolerance parameter was set to $10^{-4}$.

    \item KM~\citep{macqueen1967some}:
    Pairwise data were used for training without all link information.
    Scikit-learn implementation~\citep{scikit-learn} of $k$-means clustering was used with the default parameters.

    \item SV (Supervised):
    The true class labels were revealed during training.
    The model and optimization setup were the same as CIPS.
\end{itemize}

\paragraph{Implementation details (clustering error minimization on a real-world dataset).}
Pubmed-Diabetes dataset is a citation network dataset
consists of \num{19717} nodes representing scientific publications related to diabetes
and \num{44338} (directed) edges representing citing relationships.
Each node is described by \num{500}-dimensional TF/IDF features,
and categorized into three classes,
among which we pick class 1 (``Diabetes Mellitus, Experimental'') and 3 (``Diabetes Mellitus Type 2'') to convert it into a binary-labeled dataset.

The implementation details of our method and the baselines were as follows.
\begin{itemize}
    \setlength{\leftskip}{-10pt}
    \item CIPS (Ours):
    The \num{4}-layer perceptron (\num{500}-\num{8}-\num{8}-\num{8}-\num{1}) with the softplus activation~\citep{dugas2000incorporating} was used.
    The softmax cross entropy was optimized with Adam~\citep{kingma2015adam} (minibatch size: $\num{4096}$ / learning rate: $10^{-3}$ / training epochs: $\num{100}$).
    The $\ell_2$-regularization parameter is chosen from $\{10^{-2}, 10^{-4}, 10^{-6}\}$ by the five-fold cross-validation.
    The early stopping is applied with the patience of $\num{10}$ epochs.
    We randomly extracted $\num{20}$\% of the nodes as test data.
    The pairwise supervision was generated as follows:
    first extracted the edges whose both ends are in the training data as similar,
    then randomly coupled the non-connected nodes as dissimilar,
    with the same numbers of similar and dissimilar pairs.
    About \num{19000} pairs were obtained.

    \item MCL~\citep{hsu2018multi}:
    The setup of model, optimization, and data generation was the same as CIPS.

    \item DML~\citep{chopra2005learning}:
    The metric loss function proposed by \citet{chopra2005learning} was used.
    The model was the same as CIPS except the last layer,
    and \num{8}-dimensional outputs of the penultimate layer were used as the embeddings,
    on which $k$-means clustering was performed.
    Scikit-learn implementation~\citep{scikit-learn} of $k$-means clustering was used with the default parameters.
    The setup of optimization and data generation was the same as CIPS.

    \item SV (Supervised):
    Labeled \num{7889} nodes ($\pi_+ \approx \num{0.65}$) were used during training.
    The setup of model and optimization was the same as CIPS.
\end{itemize}

\paragraph{Full results.}
Table~\ref{table:full_tabular} shows the performance comparison with baseline methods on ten datasets from UCI and LIBSVM repositories.
Figure~\ref{fig:full_consistency} presents the sample complexity of our method on three image classification datasets including MNIST~\citep{lecun1998mnist}, Fashion-MNIST~\citep{xiao2017/online}, and Kuzushiji-MNIST~\citep{clanuwat2018deep},
where the original ten class categories were converted into positive/negative labels by grouping even/odd class labels.
Figure~\ref{fig:class_proportion} demonstrates the performance of our class assignment method with various class priors $\pi_+ \in \{0.1, 0.4, 0.7\}$.

\begin{table*}[ht]
\centering
\caption{
Mean clustering error and standard error on different benchmark datasets over $20$ trials.
Bold numbers indicate outperforming methods, chosen by one-sided t-test with the significance level $5\%$.
}
\label{table:full_tabular}
\scalebox{0.75}{
\begin{tabular}{cccccccccc}
\toprule
        dataset & \multirow{2}{*}{$m$ }&         \multirow{2}{*}{CIPS (Ours)} &         \multirow{2}{*}{MCL} &   \multirow{2}{*}{SD} &  \multirow{2}{*}{OVPC} & \multirow{2}{*}{SSP} & \multirow{2}{*}{CKM} & \multirow{2}{*}{KM} &  \multirow{2}{*}{(SV)}\\
        (dim., $\pi_+$) & \\
\midrule
adult & 100  &  39.8 (1.6) &  38.4 (2.1) &  30.8 (0.9) &  45.0 (0.9) &  \bf 24.7 (0.3) &  28.9 (0.8) &  \bf 24.9 (0.5) &  21.9 (0.4) \\
    (123, 0.24)       & 500  &  21.5 (1.0) & \bf 19.3 (0.4) &  23.2 (0.4) &  44.7 (0.9) &  24.3 (0.3) &  28.2 (0.4) &  27.5 (0.5) &  16.9 (0.3) \\
         & 1000 & \bf 17.6 (0.3) & \bf  17.2 (0.3) &  20.5 (0.3) &  45.5 (0.7) &  24.2 (0.3) &  27.9 (0.4) &  27.9 (0.5) &  15.9 (0.3) \\ \midrule
banana & 100  & \bf  43.6 (0.6) & \bf 44.5 (0.6) &  45.3 (0.6) &  46.0 (0.7) & \bf  43.0 (1.0) &  46.4 (0.7) &  45.8 (0.7) &  44.6 (0.6) \\
   (2, 0.45)       & 500  &  43.1 (0.8) &  43.3 (0.6) &  45.1 (0.7) &  46.0 (0.7) &  \bf 14.3 (0.7) &  45.5 (0.6) &  44.4 (0.4) &  45.1 (0.6) \\
         & 1000 &  44.4 (0.6) &  44.3 (0.7) &  44.4 (0.5) &  46.2 (0.5) & \bf 11.0 (0.2) &  45.0 (0.7) &  44.0 (0.3) &  45.1 (0.7) \\ \midrule
codrna & 100  & \bf 24.7 (1.8) &  32.3 (1.4) &  \bf  28.0 (1.3) &  32.0 (2.0) &  45.5 (1.5) &  46.7 (0.6) &  42.5 (1.0) &  11.0 (0.6) \\
   (8, 0.33)       & 500  &  \bf   6.4 (0.2) &  10.6 (0.3) &  12.0 (0.6) &  28.0 (2.1) &  48.6 (0.3) &  46.2 (0.3) &  44.0 (0.7) &   6.6 (0.2) \\
         & 1000 &  \bf   6.3 (0.2) &  \bf   6.5 (0.2) &   8.8 (0.4) &  28.3 (2.0) &  44.8 (1.6) &  46.1 (0.4) &  45.4 (0.6) &   6.3 (0.2) \\ \midrule
ijcnn1 & 100  &  16.6 (2.3) &  24.9 (2.9) &  \bf  10.7 (0.3) &  41.1 (1.1) &  31.6 (2.0) &  40.0 (1.3) &  31.9 (2.4) &   9.1 (0.2) \\
   (22, 0.10)       & 500  &   \bf  7.7 (0.2) &   8.2 (0.2) &   8.3 (0.2) &  41.6 (1.3) &  33.0 (2.5) &  45.4 (0.8) &  41.7 (0.7) &   7.9 (0.2) \\
         & 1000 &  \bf   7.7 (0.2) &  \bf   7.9 (0.2) &   \bf  8.1 (0.2) &  42.0 (1.4) &  34.9 (1.7) &  45.9 (0.8) &  43.4 (0.7) &   7.6 (0.2) \\ \midrule
magic & 100  &  \bf  24.9 (1.3) &  \bf  28.7 (1.8) &  30.7 (1.3) &  41.9 (1.0) &  47.1 (0.5) &  45.5 (1.2) &  44.0 (1.2) &  21.8 (0.4) \\
   (10, 0.35)       & 500  &  \bf  21.5 (0.3) &  \bf  21.3 (0.3) &  25.5 (0.8) &  39.6 (1.5) &  46.8 (0.5) &  46.8 (0.4) &  44.4 (0.4) &  20.8 (0.3) \\
         & 1000 &  \bf  21.3 (0.3) &  \bf  20.9 (0.3) &  23.8 (0.4) &  39.5 (1.7) &  43.6 (0.9) &  46.8 (0.3) &  44.6 (0.4) &  20.7 (0.3) \\ \midrule
phishing & 100  &  \bf  12.7 (2.3) &  \bf  12.8 (2.3) &  34.6 (1.8) &  41.7 (1.0) &  46.6 (0.5) &  24.4 (3.4) &  47.0 (0.5) &   7.6 (0.2) \\
  (44, 0.68)       & 500  &   7.2 (0.2) &  \bf   6.6 (0.1) &  26.9 (1.4) &  42.9 (0.8) &  46.0 (0.5) &  16.9 (2.6) &  46.4 (0.5) &   6.5 (0.2) \\
         & 1000 &  \bf   6.5 (0.2) &  \bf   6.3 (0.2) &  22.0 (1.0) &  43.8 (1.1) &  45.5 (0.5) &  15.2 (2.7) &  46.4 (0.5) &   6.3 (0.2) \\ \midrule
phoneme & 100  &   \bf 28.2 (1.2) &  33.1 (1.9) &  \bf  29.1 (1.2) &  38.4 (1.3) &  31.0 (1.3) &  \bf  28.0 (1.0) &  32.9 (1.2) &  25.7 (0.4) \\
    (5, 0.71)      & 500  &  \bf  25.0 (0.4) &  \bf  24.2 (0.5) &  26.1 (0.6) &  38.6 (1.9) &  25.5 (0.5) &  28.0 (0.8) &  32.7 (0.3) &  25.0 (0.3) \\
         & 1000 &  \bf  25.2 (0.4) &  \bf  25.0 (0.4) &  26.0 (0.4) &  39.8 (1.5) &  \bf  24.5 (0.5) &  30.2 (0.6) &  32.7 (0.3) &  25.3 (0.2) \\ \midrule
spambase & 100  &  \bf  13.8 (1.0) &  \bf  13.3 (1.3) &  31.6 (1.5) &  39.7 (1.3) &  40.5 (0.4) &  \bf  15.9 (2.0) &  39.7 (1.3) &  10.5 (0.3) \\
   (57, 0.39)        & 500  &   9.4 (0.2) &   \bf  8.6 (0.2) &  22.6 (0.9) &  38.0 (1.6) &  40.8 (0.3) &  11.5 (0.2) &  37.4 (2.3) &   8.5 (0.2) \\
         & 1000 &   8.3 (0.2) &  \bf   7.6 (0.1) &  19.7 (0.8) &  39.3 (1.2) &  40.2 (0.4) &  11.5 (0.2) &  39.7 (1.3) &   7.8 (0.2) \\ \midrule
w8a & 100  &  31.5 (1.9) &  31.4 (2.1) &  11.8 (0.3) &  39.7 (1.4) &  \bf   5.3 (1.2) &   \bf  6.8 (1.9) &  \bf   5.5 (1.3) &  10.3 (0.4) \\
    (300, 0.03)      & 500  &   5.6 (0.7) &   4.2 (0.5) &  \bf   3.2 (0.1) &  38.3 (1.3) &  \bf   3.5 (0.1) &  14.0 (3.1) &   5.5 (1.1) &   2.6 (0.1) \\
         & 1000 &   2.6 (0.2) &  \bf   2.2 (0.1) &   2.6 (0.2) &  43.1 (0.8) &   3.0 (0.1) &   8.9 (2.6) &   3.7 (0.5) &   2.0 (0.1) \\ \midrule
waveform & 100  &  \bf  18.2 (0.3) &  \bf  17.7 (0.3) &  26.4 (0.9) &  41.9 (1.6) &  44.1 (0.6) &  41.0 (1.3) &  45.1 (0.6) &  16.2 (0.2) \\
   (21, 0.33)      & 500  &  15.8 (0.2) &  \bf  15.1 (0.2) &  20.2 (0.5) &  38.9 (1.3) &  44.9 (0.7) &  45.1 (0.6) &  47.1 (0.4) &  14.8 (0.2) \\
         & 1000 &  \bf  14.9 (0.2) &  \bf  14.7 (0.2) &  18.4 (0.3) &  37.0 (1.7) &  45.5 (0.5) &  44.9 (0.5) &  47.8 (0.4) &  14.4 (0.2) \\
\bottomrule
\end{tabular}
}
\end{table*}

\begin{figure}[ht]
  \centering
  \subfigure[MNIST]{\includegraphics[height=4cm]{contents/img/consistency/consistency_mnist.pdf}}
  \subfigure[Fashion-MNIST]{\includegraphics[height=4cm]{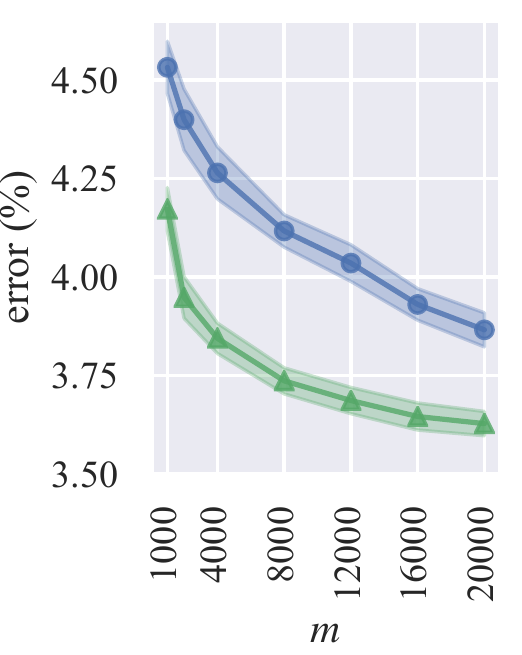}}
  \subfigure[Kuzushiji-MNIST]{\includegraphics[height=4cm]{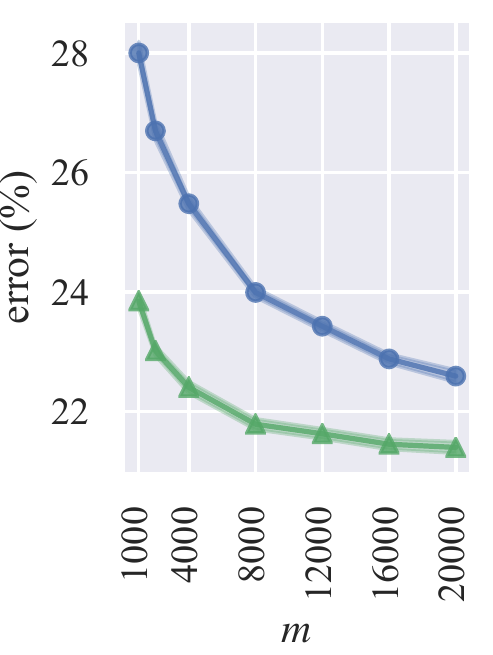}}
  \caption{
    Mean clustering error and standard error (shaded areas) over $20$ trials on image classification datasets.
  }
  \label{fig:full_consistency}
\end{figure}

\begin{figure}[ht]
  \centering
  \subfigure[MNIST]{\includegraphics[width=0.25\linewidth]{contents/img/priors/prior_mnist.pdf}}
  \subfigure[Fashion-MNIST]{\includegraphics[width=0.25\linewidth]{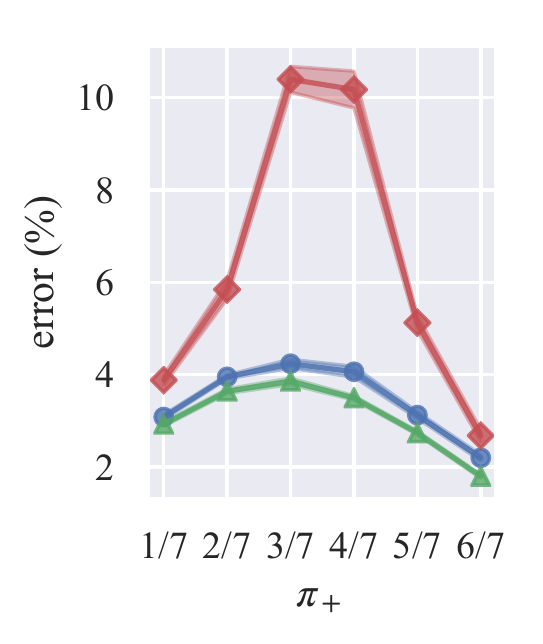}}
  \subfigure[Kuzushiji-MNIST]{\includegraphics[width=0.25\linewidth]{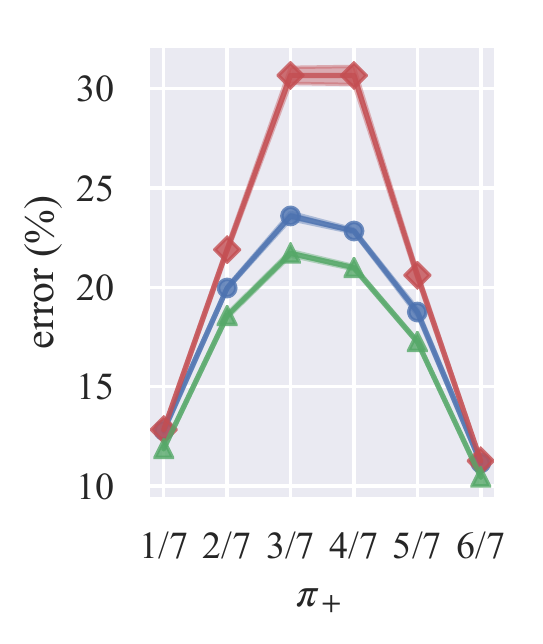}}
  \caption{
    Mean clustering error and standard error (shaded areas) over ten trials on image classification datasets under controlled class priors.
  }
  \label{fig:class_proportion}
\end{figure}

\begin{figure}[ht]
  \centering
  \subfigure[$\pi_+ =0.1$]{\includegraphics[width=0.3\linewidth]{contents/img/class_assignment/class_assinment_01.pdf}}
  \subfigure[$\pi_+ =0.4$]{\includegraphics[width=0.3\linewidth]{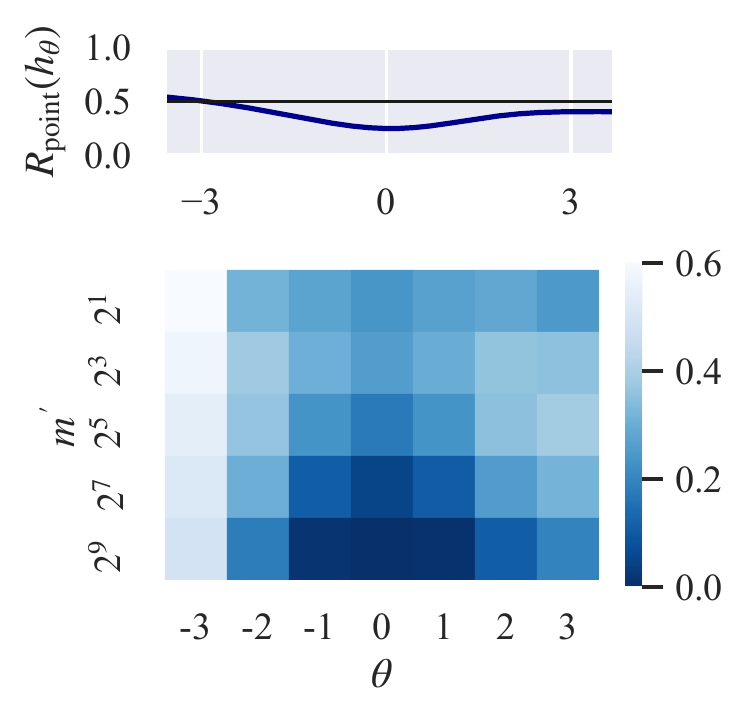}}
  \subfigure[$\pi_+ =0.7$]{\includegraphics[width=0.3\linewidth]{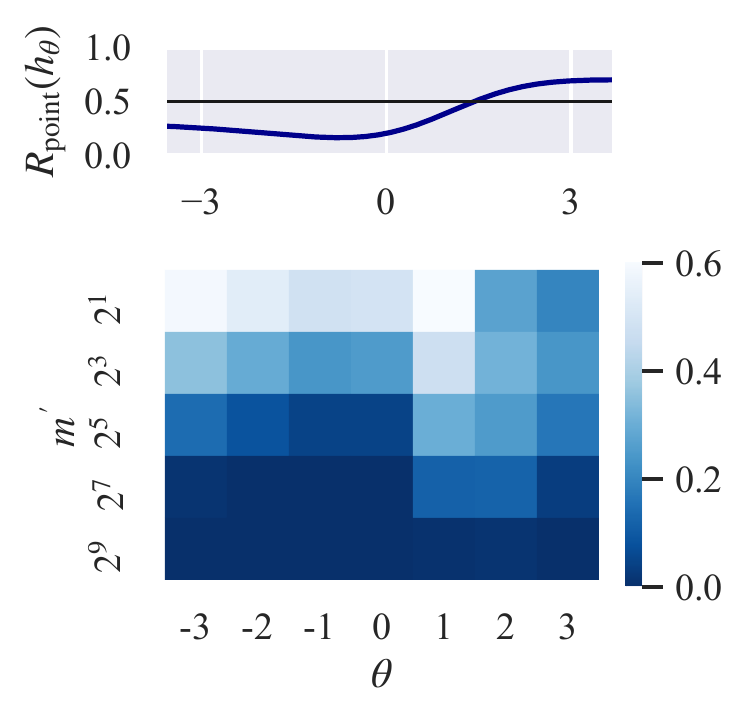}}
  \caption{
    Classification error for each threshold classifier (upper) and
    the error probability of the proposed class assignment method over $\num{10000}$ trials (bottom) on the synthetic Gaussian dataset  with $\pi_+ \in \{0.1, 0.4, 0.7 \}$.
    The detail of the dataset is described in Section~\ref{sec:experiments}.
  }
  \label{fig:full_class_assignment}
\end{figure}

\end{document}